\documentclass[journal]{IEEEtran}
\usepackage{amsmath,amsfonts}
\usepackage{algorithmic}
\usepackage{algorithm}
\usepackage{array}
\usepackage[caption=false,font=scriptsize,labelfont=sf,textfont=sf]{subfig}
\usepackage{textcomp}
\usepackage{stfloats}
\usepackage{url}
\usepackage{bbm}
\usepackage{verbatim}
\usepackage{graphicx}
\usepackage{cite}
\usepackage[normalem]{ulem}
\usepackage{siunitx}
\usepackage{nicefrac}
\usepackage{amsthm}
\usepackage[dvipsnames]{xcolor}
\usepackage[hidelinks]{hyperref}
\usepackage{multirow}
\usepackage{multicol}
\usepackage{soul}

\newtheorem{definition}{Definition}
\newtheorem{proposition}{Proposition}

\usepackage{tikz}
\usetikzlibrary{shapes,arrows}
\tikzstyle{block} = [rectangle, fill, fill=teal!10, 
    text width=12em, text centered, minimum height=2em]
\tikzstyle{empty} = [rectangle, fill, fill=white, 
     text centered]
\tikzstyle{line} = [draw, -latex']
\usetikzlibrary{calc}
\usetikzlibrary{positioning,fit,arrows.meta,backgrounds,calc}
\usetikzlibrary{decorations.pathmorphing}

\tikzset{
    algorithm/.style={%
        draw, rounded corners,
        minimum width=0cm,
        minimum height=1cm,
        font={\fontsize{10pt}{12}\sffamily},
        fill=white,
        text=black,
        },
    ours/.style={
        draw, rounded corners,
        minimum width=0cm,
        minimum height=1cm,
        font={\fontsize{10pt}{12}\sffamily},
        fill=gray,
        text=white,
    },
    module/.style={%
        draw, rounded corners,
        minimum width=1.5cm,
        minimum height=0.5cm,
        font={\fontsize{7pt}{12}\sffamily},
        fill=teal,
        text=white
        },
    input/.style={
        draw, rounded corners,
        minimum width=1.5cm,
        minimum height=0.5cm,
        font={\fontsize{7pt}{12}\sffamily},
        fill=violet,
        text=white
    },
    output/.style={
        draw, rounded corners,
        minimum width=1.5cm,
        minimum height=0.5cm,
        font={\fontsize{7pt}{12}\sffamily},
        fill=brown,
        text=white
    },
    fixed/.style={
        draw, rounded corners,
        minimum width=1.5cm,
        minimum height=0.5cm,
        font={\fontsize{7pt}{12}\sffamily},
        fill=gray,
        text=white
    },
}

\newcounter{experimentcounter}
\newcommand{\normtwo}[1]{\left\lVert#1\right\rVert_2}

\newcommand{\mR}{\mathcal{R}} 

\newcommand{\mRxy}[2]{\mathcal{R}_{#1,#2}}
\newcommand{\mO}{\mathcal{O}}
\newcommand{\mC}{\mathcal{C}}
\newcommand{\mH}{\mathcal{H}}

\newcommand{\mD}{\mathcal{D}}

\newcommand{\mB}{\mathcal{B}}
\newcommand{\mS}{\mathcal{S}}
\newcommand{\mP}{\mathcal{P}}

\newcommand{\mQ}{\mathcal{Q}}

\newcommand{\mJ}{\mathcal{J}}

\newcommand{\mM}{\mathcal{M}}

\newcommand{\mI}{\mathcal{I}}

\newcommand{\vp}{\mathbf{p}}

\newcommand{\vv}{\mathbf{v}}

\newcommand{\vf}{\mathbf{f}}

\newcommand{\vd}{\mathbf{d}}

\newcommand{\vDelta}{\boldsymbol{\Delta}}

\newcommand{\vx}{\mathbf{x}}

\newcommand{\vg}{\mathbf{g}}

\newcommand{\vD}{\mathbf{D}}
\newcommand{\vzero}{\mathbf{0}}

\newcommand{\vc}{\mathbf{c}}

\newcommand{\vone}{\mathbf{1}}

\newcommand{\vlambda}{\boldsymbol{\lambda}}

\newcommand{\vP}{\mathbf{P}}
\newcommand{\vr}{\mathbf{r}}

\newcommand{\ie}{\emph{i.e.,} } 
\newcommand{\eg}{\emph{e.g.,} } 

\begin{document}

\title{DREAM: Decentralized Real-time Asynchronous Probabilistic Trajectory Planning for Collision-free Multi-Robot Navigation in Cluttered Environments}

\author{Bask{\i}n \c{S}enba\c{s}lar and
        Gaurav S. Sukhatme
\thanks{Bask{\i}n \c{S}enba\c{s}lar (corresponding author) is with NVIDIA and Gaurav S. Sukhatme is with the Department of Computer Science, University of Southern California, Los Angeles, CA, USA. GSS holds concurrent appointments as a Professor at USC and as an Amazon Scholar. This paper describes work performed at USC when B\c{S} was a PhD student at USC and is not associated with Amazon or NVIDIA. Email: \{baskin.senbaslar, gaurav\}@usc.edu.}
}

\markboth{}%
{}

\IEEEpubid{}

\maketitle

\begin{abstract}

Collision-free navigation in cluttered environments with static and dynamic obstacles is essential for many multi-robot tasks.
Dynamic obstacles may also be interactive, \ie their behavior varies based on the behavior of other entities.
We propose a novel representation for interactive behavior of dynamic obstacles and a decentralized real-time multi-robot trajectory planning algorithm allowing inter-robot collision avoidance as well as static and dynamic obstacle avoidance.
Our planner simulates the behavior of dynamic obstacles, accounting for interactivity.
We account for the perception inaccuracy of static and prediction inaccuracy of dynamic obstacles.
We handle asynchronous planning between teammates and message delays, drops, and re-orderings.
We evaluate our algorithm in simulations using 25400 random cases and compare it against three state-of-the-art baselines using 2100 random cases. 
Our algorithm achieves up to 1.68x success rate using as low as 0.28x time in single-robot, and up to 2.15x success rate using as low as 0.36x time in multi-robot cases compared to the best baseline.
We implement our planner on real quadrotors to show its real-world applicability.\looseness=-1
\end{abstract}

\begin{IEEEkeywords}
collision avoidance, multi-robot systems, motion and path planning, probabilistic trajectory planning
\end{IEEEkeywords}
\section*{Supplemental Video}
\begin{center}
\url{https://youtu.be/ct8okY5pmgI}
\end{center}

\section {Introduction}

Collision-free mobile robot navigation in cluttered environments is a foundational problem in settings such as autonomous driving~\cite{campbell2010autonomous}, autonomous last-mile delivery~\cite{li2020lastmile}, and warehouse automation~\cite{inam2018warehouse}.
In such environments, obstacles can be static or dynamic. Further, dynamic obstacles may be interactive, \ie changing their behavior according to the behavior of other entities.
There can be multiple mobile robots explicitly cooperating with each other to avoid collisions.
Here, we present DREAM: a \underline{d}ecentralized \underline{re}al-time \underline{a}synchronous probabilistic trajectory planning algorithm for \underline{m}ulti-robot teams (Fig.~\ref{Figure:ClutteredEnvironment}).\looseness=-1

Each robot uses onboard sensing to perceive its environment and classifies objects into three sets: static obstacles, dynamic obstacles, and teammates.
It produces a probabilistic representation of static obstacles, in which each static obstacle has an existence probability.
Each robot uses an onboard prediction system to predict the behaviors of dynamic obstacles and assigns realization probabilities to each behavior.
The perception system provides the current shapes of teammates, i.e., we require geometry-only sensing for teammates and do not require estimation/communication of higher-order derivatives, e.g., velocities or accelerations.
Each robot computes discretized separating hyperplane trajectories (DSHTs)~\cite{senbaslar2022async} between itself and teammates, and uses DSHTs during decision-making for inter-teammate collision avoidance, allowing safe operation under asynchronous planning and imperfect communication.\looseness=-1

\begin{figure}
    \centering
    \includegraphics[width=\linewidth]{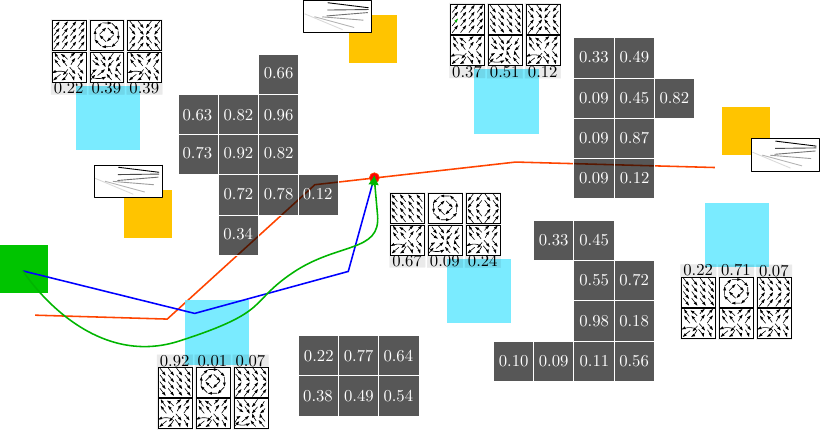}
    \caption{Static obstacles ({\color{darkgray}gray}) are modeled using their shapes and existence probabilities. Dynamic obstacles ({\color{cyan}cyan}) are modeled using their shapes, current positions, and a probability distribution over their behavior models, each of which comprises a movement and an interaction model. Teammates ({\color{Peach}orange}) are modeled using discretized separating hyperplane trajectories (DSHTs). The planner selects a goal position on a desired trajectory ({\color{red}red}), plans a spatiotemporal discrete path ({\color{blue}blue}) to the goal position while minimizing the probability of collision with static and dynamic obstacles and DSHT violations, and solves a quadratic program to fit a smooth trajectory ({\color{ForestGreen}green}) to the discrete plan while preserving the collision probabilities computed and DSHT hyperplanes not violated during search.}
    \label{Figure:ClutteredEnvironment}
    \vspace{-0.3in}
\end{figure}

Using these uncertain representations of static and dynamic obstacles and DSHTs for teammates, each robot generates dynamically feasible polynomial trajectories in real-time by primarily minimizing the probabilities of collisions with static and dynamic obstacles and DSHT violations, while minimizing distance, duration, rotation, and energy usage as secondary objectives using our planner.
A DSHT hyperplane is said to be violated when the robot is not fully contained in the safe side of the hyperplane.
During decision-making, we consider interactive behaviors of dynamic obstacles in response to robot actions.
The planner runs in a receding horizon fashion, in which the planned trajectory is executed for a short duration and a new trajectory is planned from scratch.
The planner can be guided with desired trajectories, therefore it can be used in conjunction with offline planners that perform longer horizon decision-making.\looseness=-1

DREAM utilizes a three-stage widely used pipeline~\cite{senbaslar2023rlss, senbaslar2021rlss, tordesillas2020mader,kondo2023rmader,liu2017planning, richter2013planning,chen2016planning,gao2018fast}, differing in specific operations from prior work at each stage:\looseness=-1
\begin{enumerate}
\item \textbf{Goal Selection:} Choose a goal position on the desired trajectory to plan to and the time at which the goal position should be (or should have been) reached,\looseness=-1
\item \textbf{Discrete Search:} Find a discrete spatiotemporal path to the goal that minimizes the probability of collision with static and dynamic obstacles, DSHT violations, and total duration, distance, and the number of rotations,\looseness=-1
\item \textbf{Trajectory Optimization:} Solve a quadratic program (QP) to safely fit a dynamically feasible trajectory to the discrete plan while preserving i) the collision probabilities computed and ii) DSHT elements not violated during search.\looseness=-1
\end{enumerate}


The contributions of our work are as follows:\looseness=-1
\begin{itemize}
    \item We introduce a simple representation for interactive behaviors of dynamic obstacles that can be used within a planner, enabling efficient forward simulation of multiple futures.\looseness=-1
    \item We propose a decentralized real-time trajectory planning algorithm for multi-robot navigation in cluttered environments that produces dynamically feasible trajectories avoiding static and (interactive) dynamic obstacles and teammates that plan asynchronously. \looseness=-1
    Our algorithm handles message delays, drops, and re-orderings between teammates.
    It explicitly accounts for sensing uncertainty with static obstacles and prediction uncertainty with dynamic obstacles.
    \item We evaluate our algorithm extensively in simulations to show its performance under different environments and configurations using 25400 randomly generated runs.
    We compare its performance to three state-of-the-art multi-robot navigation decision-making algorithms using 2100 randomly generated runs, and show that our algorithm achieves up to 1.68x success rate using as low as 0.28x time in the single-robot case, and 2.15x success rate using as low as 0.36x time in multi-robot scenarios compared to the best baseline.
    We implement our algorithm for physical quadrotors and show its feasibility in the real world.\looseness=-1
\end{itemize}

\section{Related Work}\label{Section:RelatedWork}

\textbf{Static and dynamic obstacle avoidance and accounting for uncertainty:} Various approaches for avoiding static and dynamic obstacles and integrating uncertainty associated with several sources (\eg unmodeled system dynamics, state estimation inaccuracy, perception noise, or prediction inaccuracies) have been proposed. 
\cite{tordesillas2020mader} proposes a polynomial trajectory planner to avoid static obstacles and dynamic obstacles given their predicted trajectories along with a maximum prediction error. 
\cite{qi2023unstruc} combines motion primitive search with spline optimization for static and dynamic obstacle avoidance. Chance constrained RRT (CC-RRT)~\cite{luders2010chance} plans trajectories to avoid static and dynamic obstacles, limiting the probability of collisions under Gaussian system and prediction noise.
\cite{aoude2013probabilistically} performs trajectory prediction using Gaussian mixture models to estimate motion models of dynamic obstacles, and uses these models within an RRT variant to predict their trajectories as a set of trajectory particles within CC-RRT to compute and limit collision probabilities.
\cite{zhu2019chance} proposes a chance-constrained MPC formulation for static and dynamic obstacle avoidance where uncertainty stems from Gaussian system model and state estimation noise, and dynamic obstacle model noise where dynamic obstacles are modeled using constant velocities with Gaussian acceleration noise.
RAST~\cite{chen2023rast} is a risk-aware planner that does not require segmenting obstacles into static and dynamic, but uses a particle-based occupancy map in which each particle is associated with a predicted velocity; and~\cite{nair2022collision} an MPC-based collision avoidance method where uncertainty stems from system noise of the robot and prediction noise for dynamic obstacles.
\cite{janson2018monte} uses a Monte Carlo sampling to compute collision probabilities of trajectories under system uncertainty.\looseness=-1

Prior decentralized decision-making approaches have been proposed for the cooperative navigation of multiple robots, in which each robot computes its own plan, cooperating with others during decision-making for collision avoidance.
We classify them into two groups: short and medium horizon decision-making algorithms, where the algorithms in the former output a single action to execute, while the algorithm in the latter output medium horizon trajectories, e.g., trajectories that are $2-10$ seconds long, in a receding horizon fashion.
Our approach falls into the latter category.\looseness=-1

\textbf{Short horizon multi-robot decision making:}~\cite{van2011reciprocal} presents optimal reciprocal collision avoidance (ORCA), a velocity obstacle-based approach, which outputs velocity commands.
\cite{wang2017safety} utilizes safety barrier certificates (SBC) for collision avoidance, which outputs acceleration commands.
GLAS~\cite{riviere2020glas} combines a learned network trained to imitate a global planner~\cite{honig2018quadswarms} with a safety module to generate safe actions.
\cite{nn2021batra} proposes using a learned network to control the thrusts of quadrotor propellers.
Several approaches to solve multi-agent path finding problems on grids using learned networks with direction outputs have also been proposed~\cite{damani2020primal2,li2020gnn}.\looseness=-1

\textbf{Medium horizon multi-robot decision making:}
\cite{zhou2017bvc} utilizes buffered Voronoi cells (BVC) within a model predictive control (MPC) framework, where each robot stays within its cell in each planning iteration.
BVC requires position-only sensing and does not depend on inter-robot communication.
~\cite{wang2021dpmc} presents a distributed model predictive control (DMPC) scheme that requires full state sensing  between robots.
Utilizing a DMPC scheme with full plan communication is also proposed~\cite{luis2019dmpc,luis2020dmpc}.
Accounting for asynchronous planning between robots becomes essential when planning durations increase.
\cite{senbaslar2022async} introduces discretized separating hyperplane trajectories (DSHTs) as a constraint generation mechanism to account for asynchronous planning under imperfect inter-robot communication, and extends BVC planner with the DSHTs to adopt it to asynchronous planning scenarios.
Differential flatness~\cite{murray1995differential} of the underlying systems is utilized to plan in the output space instead of the input space by many planners, which allows planning continuous splines with limited derivative magnitudes to ensure dynamic feasibility.
RTE~\cite{senbaslar2018rte} uses buffered Voronoi cells in a spline optimization framework and combines the optimization with discrete planning to locally resolve deadlocks.
Obstacle avoidance is ensured using safe navigation corridors (SNC) during optimization.
RLSS~\cite{senbaslar2023rlss} uses support vector machines instead of Voronoi diagrams to support robots with any convex shape and ensures kinematic feasibility of the generated problem.
MADER~\cite{tordesillas2020mader} combines discrete planning with spline optimization, treating SNC constraints as decision variables in a non-linear optimization problem.
It explicitly accounts for asynchronous planning using communication, while assuming instantaneous perfect communication between robots.
RMADER~\cite{kondo2023rmader} extends MADER to handle communication delays with known bounds between teammates.
RSFC~\cite{park2020rsfc,park2021rsfc} plans for piecewise splines with B\'ezier curve~\cite{farouki2012bernstein} pieces where safety between robots is ensured by making sure that their relative trajectories are in a safe set, where trajectories between robots are shared with instantaneous perfect communication.
LSC~\cite{park2022lsc} extends RSFC by using linear safety constraints without slack variables, which may cause the final solution to be unsafe in RSFC.
Ego-swarm~\cite{zhou2021ego} formulate collision avoidance as a cost function, which they optimize using gradient-based local optimization.
TASC~\cite{toumieh2022tasc,toumieh2023tasc} uses SNCs, which it computes between the communicated plans of other robots and the last plan of the planning robot.
TASC accounts for bounded communication delays.\looseness=-1

\textbf{Prediction of dynamical systems:} Predicting future states of dynamical systems is studied extensively and many recent approaches have been developed in the autonomous vehicle domain. 
\cite{wiest2012pred} uses Gaussian mixture models to estimate a Gaussian distribution over the future states of a vehicle given its past states.
\cite{lee2017desire} predicts future trajectories of dynamic obstacles by learning a posterior distribution over future dynamic obstacle trajectories given past trajectories.
Multi-modal prediction for vehicles to tackle bias against unlikely future trajectories during training is also investigated~\cite{kim2022diverse}.
\cite{bartoli2018context} presents a method for human movement prediction using context information modeling human-human and human-static obstacle interactions.
\cite{zhou2023dyn} generates multi-modal pedestrian predictions utilizing and modeling social interactions between humans and human intentions. 
State-of-the-art approaches that predict future trajectories of dynamic obstacles given past observations, potentially in a multi-modal way, use relatively computationally heavy approaches making them hard to re-query to model interactivity between the robot and dynamic obstacles during decision-making. 
In this paper, we propose \emph{policies} that are fast to query as prediction outputs instead of future trajectories.
Policies model intentions of the dynamic obstacles (movement models) as well as the interaction between dynamic obstacles and the robot (interaction models) as vector fields of velocities.\looseness=-1

\textbf{Novelty.} Compared to the listed multi-robot planning literature, DREAM is the only planner that explicitly models and accounts for interactivity of dynamic obstacles during decision making. In addition, it is the only approach that explicitly models multiple behavior hypotheses for dynamic obstacle behaviors and accounts for uncertainty across them.
Compared to short horizon approaches~\cite{alonsomora2013orca,wang2017safety,riviere2020glas,nn2021batra}, DREAM allows superior deadlock resolution as it reasons about longer horizon. In Section~\ref{Section:BaselineComparison} we show that this is the case for~\cite{wang2017safety}.
Compared to existing medium horizon approaches, DREAM allows dynamic obstacle avoidance unlike~\cite{zhou2017bvc, wang2021dpmc, luis2019dmpc, luis2020dmpc, senbaslar2018rte, senbaslar2023rlss, senbaslar2022async}. 
It explicitly accounts for static obstacle sensing uncertainty unlike all existing medium horizon approaches.
DREAM utilizes DSHTs introduced in~\cite{senbaslar2022async} for multi-robot collision avoidance which allows inter-robot collision avoidance under asynchronous planning, unbounded communication delays and message drops, which no other algorithm except DREAM and~\cite{senbaslar2022async} can provide.
Compared to~\cite{senbaslar2022async}, DREAM allows static and dynamic obstacle avoidance, and models sensing uncertainty of static and prediction uncertainty of dynamic obstacles.
We compare DREAM against existing medium horizon approaches in Section~\ref{Section:BaselineComparison}.\looseness=-1

\section{Problem Definition}\label{Section:Problem}

Consider a team of $\#^R$ robots.
Let $\mR_i^{robot}: \mathbb{R}^d \rightarrow P(\mathbb{R}^d)$ be the convex set-valued collision shape function of robot $i$, where $i \in \{1, \ldots, \#^R\}$ and $\mR_i^{robot}(\vp)$ is the space occupied by the robot when placed at position $\vp$.
Here, $d \in \{2,3\}$ is the ambient dimension that the robots operate in, and $P(\mathbb{R}^d)$ is the power set of $\mathbb{R}^d$.
We assume that the robots are rigid, and the collision shape functions are defined as $\mR_i^{robot}(\vp) = \mRxy{i}{\vzero}^{robot} \oplus \{\vp\}$ where $\mRxy{i}{\vzero}^{robot}$ is the shape of robot $i$ when placed at the origin $\vzero$ and $\oplus$ is the Minkowski sum operator.\looseness=-1

We assume that the robots are differentially flat~\cite{murray1995differential}, i.e., their states and inputs can be expressed in terms of their output trajectories and their finite derivatives, and the output trajectory is the Euclidean trajectory that the robot follows.
When a system is differentially flat, its dynamics can be accounted for by imposing output trajectory continuity up to the required degree of derivatives and imposing constraints on maximum derivative magnitudes.
Many existing systems like quadrotors~\cite{mellinger2011snap} or car-like robots~\cite{murray1993CarLike} are differentially flat.
Each robot $i$ requires output trajectory continuity up to degree $c_i$, and has maximum derivative magnitudes $\gamma_i^k$ for derivative degrees $k \in \{1, \ldots, K_i\}$ where $K_i$ is the degree up to which $i^{th}$ robot has a derivative magnitude limit.\looseness=-1

Each robot $i$ detects objects and classifies them into three sets: static obstacles $\mO_i$, dynamic obstacles $\mD_i$, and teammates $\mC_i$.
Static obstacles do not move.
Dynamic obstacles move with or without interaction with the robot.
Teammates are other robots that navigate executing the output of our planner.\looseness=-1

Each static obstacle $j \in \mO_i$ has a convex shape $\mQ_{i, j} \subset \mathbb{R}^d$, and has an existence probability $p_{i, j}^{stat} \in [0, 1]$.
Many existing data structures including occupancy grids~\cite{homm2010efficient} and octrees~\cite{hornung2013octomap} support storing obstacles in this form, in which each occupied cell is considered a separate obstacle with its existence probability.
Each perceived teammate $j\in \mC_i$ has a convex shape $\mS_{i, j}$ sensed by robot $i$.\looseness=-1

Each dynamic obstacle $j\in\mD_i$ is modeled using i) its current position $\vp^{dyn}_{i, j}$, ii) its convex set valued collision shape function $\mR_{i, j}^{dyn}: \mathbb{R}^d \rightarrow P(\mathbb{R}^d)$ where $\mR_{i,j}^{dyn}(\vp) = \mR_{i,j,\vzero}\oplus \{\vp\}$ and $\mR_{i, j, \vzero}$ is the shape of obstacle $j$ when placed at the origin, and iii) a probability distribution over its $\#_{i,j}^B$ predicted behavior models $\mB_{i, j, k}$, $k\in\{1, \ldots, \#_{i,j}^B\}$, where each behavior model is a 2-tuple $\mB_{i, j, k} = (\mM_{i, j, k}, \mI_{i, j, k})$ such that 
$\mM_{i, j, k}$ is the movement and $\mI_{i, j, k}$ is the interaction model of the dynamic obstacle.
$p^{dyn}_{i,j,k}$ is the probability that dynamic obstacle $j$ moves according to behavior model $\mB_{i, j, k}$ such that $\sum_{k=1}^{\#_{i,j}^B}p^{dyn}_{i, j, k} \leq 1$ for all $j \in \mD_i$.\looseness=-1

A movement model $\mM: \mathbb{R}^d \rightarrow \mathbb{R}^d$ is a function from the dynamic obstacle's position to its desired velocity, describing its intent.
An interaction model $\mI: \mathbb{R}^{4d} \rightarrow \mathbb{R}^d$ is a function describing robot-dynamic obstacle interaction of the form $\vv^{dyn} = \mI(\vp^{dyn}, \tilde{\vv}^{dyn}, \vp^{robot}, \vv^{robot})$.
Its arguments are 4 vectors: position $\vp^{dyn}$ of the dynamic obstacle, desired velocity $\tilde{\vv}^{dyn}$ of the dynamic obstacle (which is obtained from the movement model, i.e., $\tilde{\vv}^{dyn} = M(\vp^{dyn})$), and position $\vp^{robot}$ and velocity $\vv^{robot}$ of robot.
It outputs the velocity $\vv^{dyn}$ of the dynamic obstacle.
Notice that interaction models do not model interactions between multiple dynamic obstacles or interactions with multiple teammates, i.e., the velocity $\vv^{dyn}$ of a dynamic obstacle does not depend on the position or velocity of other dynamic obstacles or the other teammates from the perspective of a single teammate.
This is an accurate representation in sparse environments where moving objects are not in close proximity to each other but an inaccurate assumption in dense environments.
We choose to model interactions this way for computational efficiency as well as non-reliance on perfect communication: modeling interactions between multiple dynamic obstacles would result in a combinatorial explosion of possible dynamic obstacle behaviors since we support multiple hypotheses for each dynamic obstacle, and modeling interactions of dynamic obstacles with multiple teammates would require joint planning for all robots, requiring perfect communication\footnote{One could also define a single joint interaction model for all dynamic obstacles and perform non-probabilistic decision making with respect to them if inter-dynamic obstacle interactions exist and single dynamic obstacle models are insufficient at describing the dynamic obstacle behaviors.}.
While using only position and velocity to model robot-dynamic obstacle interaction is an approximation of reality, we choose this model because of its simplicity.
This simplification allows us to use interaction models to update the behavior of dynamic obstacles during discrete search efficiently\footnote{During planning, we evaluate movement and interaction models sequentially to compute the velocity of dynamic obstacles.
One could also combine movement and interaction models and have a single function to describe the dynamic obstacle behavior for planning.
We choose to model them separately to allow separate predictions of these models.}.\looseness=-1

We model sensing uncertainty of static obstacles using existence probabilities, while we only model prediction uncertainty of dynamic obstacles and not their sensing uncertainty.
The reason stems from the practicality of using different uncertainty representations for different types of obstacles.
Modeling sensing uncertainty of static obstacles using existence probabilities is readily provided in many spatial data structures, including octrees and occupancy grids.
Reasoning about dynamic obstacles require frame-by-frame tracking.
This typically requires segmentation of dynamic obstacles from the environment and estimating their states. The uncertainty of their state is typically represented using state covariances.
There is generally no question of whether a dynamic obstacle exists or not: It exists, but we are not certain what its state is.
Therefore, utilizing the existence probability model for dynamic obstacles is not useful.
Utilizing state covariance of dynamic obstacles for sensing uncertainty is problematic for our discrete search stage.
For each state with uncertainty, both the movement and interaction models would result in different velocity vectors for each possibility, making the collision checks and collision probability computation considerably more expensive even under discretization of the state space; and the probabilistic next state computation would be intractable unless movement and interaction models have limiting structures.
Hence, we choose not to model sensing uncertainty of dynamic obstacles in this work.
This allows easy introduction of multiple dynamic obstacle behavior model prediction algorithms and integration of them to decision making without requiring any structure in their outputs.\looseness=-1

Each robot $i$ has a state estimator that estimates its output derivatives up to derivative degree $c_i$, where degree $0$ corresponds to position, degree $1$ corresponds to velocity, and so on.
If state estimation accuracy is low, the trajectories computed by the planner can be used to compute the expected derivatives in an open-loop fashion assuming perfect execution.
The $k^{th}$ derivative of robot $i$'s current position is denoted with $\vp^{self}_{i,k}$ where $k \in \{0, \ldots, c_i\}$.\looseness=-1

Each robot $i$ is tasked with following a desired trajectory $\vd_i(t): [0, T_i] \rightarrow \mathbb{R}^d$ with duration $T_i$ without colliding with obstacles.
The desired trajectory $\vd_i(t)$ can be computed by a global planner using potentially incomplete prior knowledge about obstacles.
It does not need to be collision-free with respect to static or dynamic obstacles.
If no such global planner exists, it can be set to a straight line from a start position to a goal position.\looseness=-1

\section{Preliminaries}

\subsection{Discretized Separating Hyperplane Trajectories (DSHTs)}

We utilize DSHTs~\cite{senbaslar2022async} as constraints for inter-robot collision avoidance, which allows us to enforce safety when planning is asynchronous, i.e., robots start and end planning at different time points, and the communication medium is imperfect.
We briefly reiterate the theory behind DSHTs next.\looseness=-1

Let $\Omega$ be a commutative deterministic separating hyperplane computation algorithm:
it computes a separating hyperplane between two linearly separable sets, and each call to it with the same pair of arguments results in the same hyperplane. We use hard-margin support vector machines (SVM) as $\Omega$.\looseness=-1

Let $\vf_i(t):[0, T_{cur}] \rightarrow \mathbb{R}^d$ and $\vf_j(t):[0, T_{cur}]\rightarrow \mathbb{R}^d$ be the trajectories robots $i$ and $j$ executed from navigation start time $0$ to current time $T_{cur}$, respectively.
The separating hyperplane trajectory $\mH_{i,j}:[0, T_{cur}]\rightarrow \mH^d$ between robots $i$ and $j$ induced by $\Omega$ is defined as $\mH_{i,j}(t) = \Omega(\mR_i^{robot}(\vf_i(t)), \mR_j^{robot}(\vf_j(t)))$ where $\mH^d$ is the set of all hyperplanes in $\mathbb{R}^d$.\looseness=-1

Each robot $i$ stores a tail time point variable $T^{tail}_{i, j}\leq T_{cur}$ for each other robot $j$ denoting the time point after which the hyperplanes in $\mH_{i, j}$ should be used to constrain robot $i$'s plan against robot $j$.
If robot $i$ starts planning at $T_{cur}$, it uses all hyperplanes $\mH_{i, j}(t)$ where $t\in [T^{tail}_{i, j}, T_{cur}]$ to constrain itself against robot $j$ by enforcing its trajectory to be in the safe side of each hyperplane.
When robot $i$ successfully finishes a planning iteration that started at $T_{i, start}$, meaning that it is now constrained by hyperplanes from $T^{tail}_{i, j}$ to $T_{i, start}$ on $\mH_{i, j}$ against each other robot $j$, it broadcasts its identity $i$ and $T_{i, start}$.
Robots $j$ receiving the message update their tail time points against robot $i$ by setting $T^{tail}_{j, i}=T_{i, start}$, discarding constraints, and those that do not receive it do not update their tail points, over-constraining themselves against robot $i$.
As shown in~\cite{senbaslar2022async}, this constraint discarding and over-constraining mechanism ensures that active trajectories of each pair of robots share a constraining hyperplane at all times under asynchronous planning, message delays, drops and re-orderings.\looseness=-1

Let $\mH^{active}_{i,j} = \{\mH_{i,j}(t)\ |\ t\in [T^{tail}_{i,j}, T_{cur}]\}$ be the active set of separating hyperplanes of robot $i$ against robot $j$.
There are infinitely many hyperplanes in $\mH^{active}_{i,j}$ when $T^{tail}_{i,j} < T_{cur}$.
We sample hyperplanes in $\mH^{active}_{i, j}$ using a sampling step in the time domain shared among all teammates.
Let $\tilde{\mH}^{active}_{i, j}$, which is the active DSHT of robot $i$ against robot $j$, be the finite sampling of $\mH^{active}_{i, j}$, and $\tilde{\mH}^{active}_i = \{\mH \in\tilde{\mH}^{active}_{i, j}\ |\ j\in\{1, \ldots, \#^R\}\setminus \{i\}\}$ be the set of all hyperplanes that should constraint robot $i$ on a planning iteration that starts at time $T_{cur}$.
We use hyperplanes $\tilde{\mH}^{active}_i$ during planning to enforce safety against robot teammates.\looseness=-1

Use of DSHTs for inter-robot collision avoidance entail cooperative teammates, i.e., all teammates maintain and use DSHTs for inter-robot collision avoidance.
In our case, all robots run the same algorithm, utilizing DSHTs for inter-robot collision avoidance.
Apart from utilization of DSHTs for collision avoidance, there is no other interaction between robots.
DSHTs enable all pairs of robots to share a mutually excluding separating hyperplane constraint at all times.\looseness=-1

\subsection{Cost Algebraic $\text{A}^*$ Search}\label{Preliminaries:CostAlgebraic}

In the discrete search stage, we utilize the cost algebraic $\text{A}^*$ search~\cite{edelkamp2005cost}.
Cost algebraic $\text{A}^*$ is a generalization of standard $\text{A}^*$ to a richer set of cost systems, namely cost algebras.
Here, we summarize the formalism of cost algebras from the original paper~\cite{edelkamp2005cost}.
The reader is advised to refer to the original paper for a detailed and complete description of concepts.

\begin{definition}
Let $A$ be a set and $\times: A\times A \rightarrow A$ be a binary operator. A monoid is a tuple $(A, \times, \vzero)$  if the identity element $\vzero \in A$ exists, $\times$ is associative, and $A$ is closed under $\times$.\looseness=-1
\end{definition}

\begin{definition}
Let $A$ be a set. A relation $\preceq\ \subseteq A \times A$ is a total order if it is reflexive, anti-symmetric, transitive, and total.
The least operation $\sqcup$ gives the least element of the set according to a total order, i.e., $\sqcup A = c \in A$ such that $c \preceq a\ \forall a\in A$, and the greatest operation $\sqcap$ gives the greatest element of the set according to the total order, i.e., $\sqcap A = c\in A$ such that $a \preceq c\ \forall a \in A$.\looseness=-1
\end{definition}

\begin{definition}
A set $A$ is isotone if $a \preceq b$ implies both $a \times c \preceq b \times c$ and $c \times a \preceq c \times b$ for all $a,b,c\in A$.
$a \prec b$ is defined as $a \preceq b \wedge a \neq b$.
A set $A$ is strictly isotone if $a \prec b$ implies both $a \times c \prec b \times c$ and $c\times a \prec c \times b$ for all $a,b,c\in A, c\neq \vone$ where $\vone = \sqcap A$.\looseness=-1
\end{definition}

\begin{definition}
A cost algebra is a 6-tuple $(A, \sqcup, \times, \preceq, \vone, \vzero)$ such that $(A, \times, \vzero)$ is a monoid, $\preceq$ is a total order, $\sqcup$ is the least operation induced by $\preceq$, $\vone = \sqcap A$, and $\vzero = \sqcup A$, i.e. the identity element is the least element.\looseness=-1
\end{definition}

Intuitively, $A$ is the set of cost values, $\sqcup$ is the operation used to select the best among the values, $\times$ is the operation to cumulate the cost values, $\preceq$ is the operator to compare the cost values, $\vone$ is the greatest and $\vzero$ is the least cost value as well as the identity cost value under $\times$.\looseness=-1


To support multiple objectives during search, the prioritized Cartesian product of cost algebras is defined as follows.\looseness=-1

\begin{definition}\label{Definition:PrioritizedCartesianProductOfCostAlgebras}
    The prioritized Cartesian product of cost algebras $C_1 = (A_1, \sqcup_1, \times_1, \preceq_1, \vone_1, \vzero_1)$ and $C_2 =  (A_2, \sqcup_2, \times_2, \preceq_2, \vone_2, \vzero_2)$, denoted by $C_1 \times_p C_2$ is a tuple $(A_1 \times A_2, \sqcup, \times, \preceq, (\vone_1, \vone_2), (\vzero_1, \vzero_2))$ where $(a_1, a_2) \times (b_1, b_2) = (a_1 \times_1 b_1, a_2 \times_2 b_2)$, $(a_1,a_2) \preceq (b_1,b_2)$ iff $a_1 \prec_1 b_1 \vee (a_1 = b_1 \wedge a_2 \preceq_2 b_2)$, and $\sqcup$ is induced by $\preceq$.\looseness=-1
\end{definition}

Note that, $\preceq$ in Def.~\ref{Definition:PrioritizedCartesianProductOfCostAlgebras} induces lexicographical ordering among cost algebras $C_1$ and $C_2$.\looseness=-1

\begin{proposition}\label{Proposition:CartesianProductOfCostAlgebras}
If $C_1$ and $C_2$ are cost algebras, and $C_1$ is strictly isotone, then $C_1 \times_p C_2$ is also a cost algebra.
If, in addition, $C_2$ is strictly isotone, $C_1 \times_p C_2$ is also strictly isotone.\looseness=-1
\begin{proof}
Given in~\cite{edelkamp2005cost}.
\end{proof}
\end{proposition}

Proposition~\ref{Proposition:CartesianProductOfCostAlgebras} allows one to take the Cartesian product of any number of strictly isotone cost algebras and end up with a strictly isotone cost algebra.\looseness=-1

Given a cost algebra $C = (A, \sqcup, \times, \preceq, \vone, \vzero)$, cost algebraic A* finds a lowest cost path according to $\sqcup$ between two nodes in a graph where edge costs are elements of set $A$, which are ordered according to $\preceq$ and combined with $\times$ where the lowest cost value is $\vzero$ and the largest cost value is $\vone$.
Cost algebraic A* uses a heuristic for each node of the graph, and cost algebraic A* with re-openings finds cost-optimal paths only if the heuristic is admissible.
An admissible heuristic for a node is a cost $h\in A$, which underestimates the cost of the lowest cost path from the node to the goal node according to $\preceq$.\looseness=-1

The implementation of cost algebraic $\text{A}^*$ is identical to the standard $\text{A}^*$ with overloaded comparison and addition operations between costs; and overloaded largest and lowest cost values.\looseness=-1

\section{Approach}

To follow the desired trajectories $\vd_i$ as closely as possible while avoiding collisions, we propose a decentralized real-time planner executed in a receding horizon fashion.\looseness=-1


It is assumed that perception, prediction, and state estimation systems are executed independently from the planner and produce the information described in Sec.~\ref{Section:Problem}.
DSHT computation is done asynchronously and independently from the planner, maintaining tail time points $T_{i,j}$ and providing $\tilde{\mH}^{active}_i$ to the planner, which allow robots to share a mutually exluding separating hyperplane constraint under asynchronous planning and communication imperfections.
The inputs from these systems to the planner in robot $i$ are (Fig.~\ref{Figure:PlanningPipeline}):\looseness=-1
\begin{itemize}
\item \textbf{Static obstacles}: Convex shapes $\mQ_{i, j}$ with their existence probabilities such that $p^{stat}_{i,j}$ is the probability that obstacle $j \in \mO_i$ exists.\looseness=-1
\item \textbf{Dynamic obstacles}: Set $\mD_i$ of dynamic obstacles where each dynamic obstacle $j\in \mD_i$ has the current position $\vp^{dyn}_{i,j}$, collision shape function $\mR^{dyn}_{i,j}$, and behavior models $\mB_{i, j, k}$ with corresponding realization probabilities $p^{dyn}_{i, j, k}$ where $k\in\{1,\ldots,\#^B_{i, j}\}$.\looseness=-1
\item \textbf{Active DSHTs}: Set $\tilde{H}_i^{active}$ of separating hyperplanes against all other robots.\looseness=-1
\item \textbf{Self state}: The state $\{\vp^{self}_{i,0}, \ldots, \vp^{self}_{i,c_i}\}$ of the robot.\looseness=-1
\end{itemize}

\begin{figure*}[t]
\centering
\begin{tikzpicture}[]

    \node[input] (StaticObstacles) {\shortstack{Static\\Obstacles}};
    \node[input, right=0.2cm of StaticObstacles] (DynamicObstacles) {\shortstack{Dynamic\\Obstacles}};
    \node[input, below=0.2cm of StaticObstacles] (ActiveDSHTs) {\shortstack{Active\\DSHTs}};
    \node[input, right=0.2cm of ActiveDSHTs] (SelfState) {\shortstack{Self\\State}};
    \node[module, right=1cm of SelfState, yshift=0.5cm] (GoalSelection) {\shortstack{Goal\\Selection}};
    \node[module, right=1cm of GoalSelection] (DiscreteSearch) {\shortstack{Discrete\\Search}};
    \node[module, right=1cm of DiscreteSearch] (TrajectoryOptimization) {\shortstack{Trajectory\\Optimization}};
    \node[output, right=1cm of TrajectoryOptimization] (Output) {Trajectory};
    
    \draw[->] (GoalSelection) -- (DiscreteSearch);
    \draw[->] (DiscreteSearch) -- (TrajectoryOptimization);
    \draw[->] (TrajectoryOptimization) -- (Output);
    \draw[dashed, teal] ($(GoalSelection.south west)+(-0.2,-0.2)$) rectangle ($(TrajectoryOptimization.north east)+(0.2, 0.2)$);
    \draw[dashed, violet] ($(ActiveDSHTs.south west)+(-0.1,-0.1)$) rectangle ($(DynamicObstacles.north east)+(0.1, 0.1)$);
    \draw[->] ($(SelfState.north east)+(0.1, 0.125)$) -- ($(GoalSelection.west)-(0.2,0.0)$);




\end{tikzpicture}
\vspace{-0.1in}
\caption{\textbf{Planning pipeline.} {\color{violet}Inputs (purple)}, {\color{brown}outputs (brown)}, and {\color{teal}stages (teal)} of our planning algorithm are shown. All stages are run in every planning iteration with updated inputs and a new trajectory is produced. Our algorithm is executed in a receding horizon fashion, in which, a long trajectory is planned, executed for a short duration, and a new trajectory is planned from scratch. In our experiments, we run our algorithm in $\approx \SI{5}{Hz} - \SI{20}{Hz}$.}
\label{Figure:PlanningPipeline}
\vspace{-0.2in}
\end{figure*}
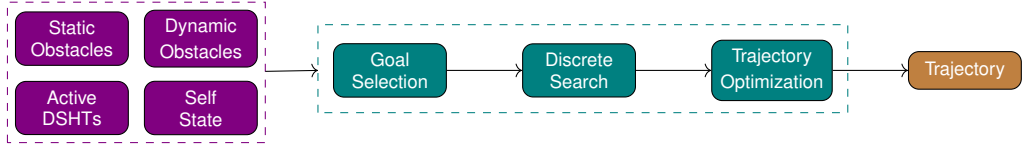

There are three stages of our algorithm (Fig.~\ref{Figure:PlanningPipeline}): i) goal selection, which selects a goal position on the desired trajectory to plan to, ii) discrete search, which computes a spatiotemporal discrete path to the goal position, minimizing the probability of collision with two classes of obstacles, DSHT violations, distance, duration, and rotations using a multi-objective search method, and iii) trajectory optimization, which safely computes a dynamically feasible trajectory by smoothing the discrete path while preserving the collision probabilities computed and DSHT hyperplanes not violated during the search.

Discrete planning needs a goal position because it utilizes a goal-directed search algorithm, which is provided by goal selection.
Then, discrete planning determines the homotopy class of the final plan in terms of collision probabilities and DSHT violations.
Last, trajectory optimization smooths the plan within the homotopy class determined by discrete planning.\looseness=-1

The planner might fail during trajectory optimization, the reasons for which are described in Sec.~\ref{Section:TrajectoryOptimization}.
If planning fails, the robot continues using its previous plan, and the best effort probabilistic collision avoidance ensured in the previous plan holds up to the accuracy of sensing and predictions.\looseness=-1

\subsection{Goal Selection}\label{Section:GoalSelection}

\begin{figure}[t]
    \centering
    \includegraphics[width=\linewidth]{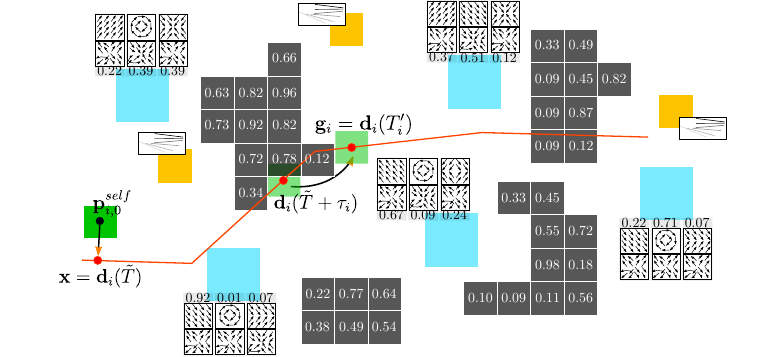}
    \vspace{-0.25in}
    \caption{\textbf{Goal selection.} The goal selection stage selects the goal position $\vg_i$ to plan to on the desired trajectory $\vd_i$ ({\color{red}red}) and the time $T_i'$ at which it should be (or should have been) reached. It finds the closest point $\vx$ on $\vd_i$ to the current robot position $\vp_{i,0}^{self}$ and its time point $\tilde{T}$, and finds the smallest time point $T_i'$ that is greater than the time point that is one desired time horizon away from $\tilde{T}$, i.e., $\tilde{T}+\tau_i$, at which the robot is collision free against all static obstacles with existence probability greater than $p_i^{min}$.}
    \label{Figure:GoalSelection}
    \vspace{-0.25in}
\end{figure}

In the goal selection stage (Fig.~\ref{Figure:GoalSelection}), each robot $i$ chooses a goal position $\vg_i$ on the desired trajectory $\vd_i$ and the time $T_i'$ at which $\vg_i$ should be (or should have been) reached. This stage has two parameters: the desired time horizon $\tau_i$ and the static obstacle existence probability threshold $p_{i}^{min}$.\looseness=-1

First, the closest point $\vx$ on the desired trajectory $\vd_i$ to the robot's current position $\vp^{self}_{i, 0}$, is found by discretizing $\vd_i$. 
Let $\tilde{T}$ be the time point of $\vx$ on $\vd_i$, i.e., $\vx=\vd_i(\tilde{T})$.
Then, goal selection finds the smallest time point $T_i' \in [\min(\tilde{T} + \tau_i, T_i), T_i]$ on $\vd_i$ such that the robot is collision-free against static obstacles with existence probabilities at least $p^{min}_i$ when placed on $\vd_i(T_i')$ using collision checks with small increments in time.
The goal position $\vg_i$ is set to $\vd_i(T_i')$, and the time at which it should be (or should have been) reached is $T_i'$. We assume that the robot placed at $\vd_i(T_i)$ is collision-free; hence such a $T_i'$ always exists.\looseness=-1

The selected goal position $\vg_i$ and the time point $T_i'$ are used during the discrete search stage, which uses a goal-directed search algorithm.
Note that goal selection chooses the goal position on the desired trajectory without considering its reachability; the actual trajectory the robot follows is planned by the rest of the algorithm.\looseness=-1

\subsection{Discrete Search}\label{Section:DiscreteSearch}

In the discrete search stage, we plan a path to the goal position $\vg_i$ using cost algebraic $\text{A}^*$ search (Section~\ref{Preliminaries:CostAlgebraic}).
We conduct a multi-objective search with six cost terms, define the cost of an action as the vector of the computed cost terms, in which each individual cost term is a strictly isotone cost algebra, and optimize over their Cartesian product, i.e., their lexicographical ordering.
Cost algebraic $\text{A}^*$ finds an optimal action sequence according to the lexicographical ordering of our cost vectors.
\looseness=-1

The individual cost terms are defined over two cost algebras, namely $(\mathbb{R}_{\geq 0} \cup \{\infty\}, min, +, \leq, \infty, 0)$, i.e., non-negative real number costs with standard addition and comparison, and $(\mathbb{N} \cup \{\infty\}, min, +, \leq, \infty, 0)$, natural numbers with standard addition and comparison, both of which are strictly isotone.
Therefore, any number of their Cartesian products are also cost algebras by Proposition~\ref{Proposition:CartesianProductOfCostAlgebras}.\looseness=-1


\begin{figure*}
    \centering
    \includegraphics[width=\linewidth]{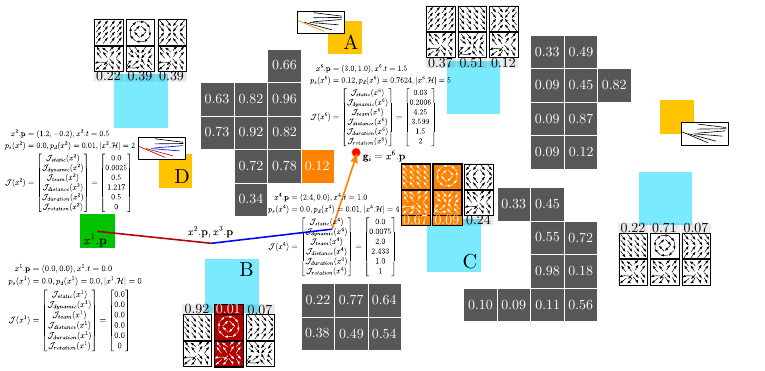}
    \vspace{-0.25in}
    \caption{\textbf{Discrete search.} A sample discrete state sequence, the associated meta data, and computed cost terms are shown. The computed state sequence has six states: $x^{1:6}$. $x^3$ and $x^5$ are expanded with ROTATE actions from $x^2$ and $x^4$ respectively, therefore, their information is not shown here to reduce clutter. The robot initially does not collide with any static or dynamic obstacle and does not violate any DSHT hyperplane at $x^1$. While traversing the first segment from $x^1$ to $x^2$ ({\color{red}red}), it collides with dynamic obstacle B's second behavior model and violates two hyperplanes in the DSHT against teammate D. While traversing the second segment from $x^3$ to $x^4$ ({\color{blue}blue}), it violates $2$ more hyperplanes from the DSHT against teammate D. While traversing the last segment from $x^5$ to $x^6$ ({\color{orange}orange}), it collides with the static obstacle with existence probability $0.12$, first two behavior models of dynamic obstacle C and violates a hyperplane in the DSHT against teammate A.}
    \label{Figure:DiscreteSearch}
    \vspace{-0.25in}
\end{figure*}

We explain the discrete search for an arbitrary robot $i$ in the team; each robot runs the same algorithm.
The planning horizon of the search is $\tau_i' = \max(\tilde{\tau}_i, T_i' - T_{cur}, \alpha_i\frac{\normtwo{\vp_{i,0}^{self} - \vg_i}}{\tilde{\gamma}_i^1})$ where $\tilde{\tau}_i$ is the minimum search horizon and $\tilde{\gamma}_i^1$ is the maximum speed parameter for the search stage. 
In other words, the planning horizon is set to the maximum of minimum search horizon, the time difference between the goal time point and the current time point, and a multiple of the minimum required time to reach the goal position $\vg_i$ from the current position $\vp_{i,0}^{self}$ applying maximum speed $\tilde{\gamma}^1_i$ where multiplier $\alpha_i \geq 1$.
The planning horizon $\tau_i'$ is used as a suggestion in the search and is exceeded if necessary, as explained later in this section.\looseness=-1

\textbf{States.} The states $x$ in our search formulation have six components: i) $x.\vp \in \mathbb{R}^d$ is the position of the state, ii) $x.\vDelta \in \{-1, 0, 1\}^d \setminus \{\vzero\}$ is the direction of the state on a grid oriented along robot's current velocity $\vp_{i,1}^{self}$ with a rotation matrix $R_{rot} \in SO(d)$ such that $R_{rot}(1, 0, \ldots, 0)^\top = \frac{\vp_{i, 1}^{self}}{\normtwo{\vp_{i,1}^{self}}}$, iii) $x.t \in [0, \infty)$ is the time of the state, iv) $x.\mO\subseteq\mO_i$ is the set of static obstacles that collide with the robot $i$ following the path from start state to $x$, v) $x.\mD$ is the set of dynamic obstacle behavior model--position pairs $(\mB_{i,j,k}, \vp_{i,j,k}^{dyn})$ such that dynamic obstacle $j$ moving according to $\mB_{i,j,k}$ does not collide the robot $i$ following the path from start state to $x$, and the dynamic obstacle ends up at position $\vp_{i,j,k}^{dyn}$, and vi) $x.\mH \subseteq \tilde{\mH}_{i}^{active}$ is the set of active DSHT hyperplanes that the robot $i$ violates following the path from start state to $x$.\looseness=-1

The start state of the search is $x^1$ with components $x^1.\vp = \vp_{i, 0}^{self}$, $x^1.\vDelta = (1, 0, \ldots, 0)^\top$, $x^1.t = 0$, $x^1.\mO$ are set of all obstacles that intersect with $\mR_i^{robot}(\vp_{i, 0}^{self})$, $x^1.\mD$ contains behavior model--position pairs $(\mB_{i, j, k}, \vp_{i,j}^{dyn})$ of dynamic obstacles $j$ that do not initially collide with robot, i.e. $\mR_i^{robot}(\vp_{i,0}^{self}) \cap \mR_{i,j}^{dyn}(\vp_{i,j}^{dyn}) = \emptyset$, one for each $k\in\{1,\ldots,\#^B_{i,j}\}$, and $x^1.\mH$ contains all hyperplanes in $\tilde{\mH}_i^{active}$ that the robot $i$ violates initially at $\vp_{i,0}^{self}$.
The goal states are all states $x^g$ with position $x^g.\vp = \vg_i$.\looseness=-1

\textbf{Actions.} There are three action types in our search. Let $x$ be the current state and $x^+$ be the state after applying an action.\looseness=-1

\begin{itemize}
\item FORWARD($s$, $t$) moves the current state $x$ to $x^+$ by applying constant speed $s$ along current direction $x.\vDelta$ for time $t$.
The state components change as follows.\looseness=-1

\begin{itemize}
    \item $x^+.\vp = x.\vp + R_{rot}\frac{x.\vDelta}{\normtwo{x.\vDelta}}st$
    \item $x^+.\vDelta = x.\vDelta$
    \item $x^+.t = x.t + t$.
    \item We compute static obstacles $\mO^+$ colliding with the robot with shape $\mR_i^{robot}$ travelling from $x.\vp$ to $x^+.\vp$ and set $x^+.\mO = x.\mO \cup \mO^+$.
    \item 
    Let $(\mB_{i,j,k}, \vp_{i,j,k}^{dyn}) \in x.\mD$ be a dynamic obstacle behavior model--position pair that does not collide with the state sequence from the start state to $x$.
    Note that robot applies velocity $\vv =  \frac{x^+.\vp - x.\vp}{x^+.t - x.t}$ from state $x$ to $x^+$.
    We get the desired velocity $\tilde{\vv}_{i,j,k}^{dyn}$ of the dynamic obstacle at time $x.t$ using its movement model: $\tilde{\vv}_{i,j,k}^{dyn} = \mM_{i,j,k}(\vp_{i,j,k}^{dyn})$.
    The velocity $\vv_{i,j,k}^{dyn}$ of the dynamic obstacle can be computed using the interaction model: $\vv_{i,j,k}^{dyn} = \mI_{i,j,k}(\vp_{i,j,k}^{dyn}, \tilde{\vv}_{i,j,k}^{dyn}, x.\vp, \vv)$.
    We check whether the dynamic obstacle shape $\mR_{i,j}^{dyn}$ swept between $\vp_{i,j,k}^{dyn}$ and $\vp_{i,j,k}^{dyn} + \vv_{i,j,k}^{dyn} t$ collides with robot shape $\mR_i^{robot}$ swept between $x.\vp$ and $x^+.\vp$.
    If not, we add not colliding dynamic obstacle behavior model by $x^+.\mD = x^+.\mD\ \cup\ \{(\mB_{i,j,k}, \vp_{i,j,k}^{dyn} + \vv_{i,j,k}^{dyn} t)\}$.
    Otherwise, we discard the behavior model.
    \item We compute the hyperplanes $\mH^+\subseteq \tilde{\mH}_i^{active}$ the robot $i$ violates at $x^+.\vp$, and set $x^+.\mH = x.\mH \cup \mH^+$.
\end{itemize}

\item ROTATE($\vDelta'$) changes the current state $x$ to $x^+$ by changing its direction to $\vDelta'$.
It is only available if $x.\vDelta \neq \vDelta'$.
The rotate action is added to penalize turns during search as discussed in the description of costs.
The state components remain the same except $x^+.\vDelta$ is set to  $\vDelta'$.\looseness=-1

\item REACHGOAL changes the current state $x$ to $x^+$ by connecting $x.\vp$ to the goal position $\vg_i$.
The remaining search horizon for the robot to reach its goal position is given by $\tau_i' - x.t$.
Recall that the maximum speed of the robot during the search is $\tilde{\gamma}^1_i$; hence
the robot needs at least $\frac{\normtwo{\vg_i - x.\vp}}{\tilde{\gamma}^1_i}$ seconds to reach the goal position from state $x$.
We set the duration of this REACHGOAL action to the maximum of these two values: $\max(\tau_i' - x.t, \frac{\normtwo{\vg_i - x.\vp}}{\tilde{\gamma}^1_i})$.
Therefore, the search horizon $\tau_i'$ is merely a suggestion during search and is exceeded whenever it is not dynamically feasible to reach the goal position within the search horizon.
The state components change as follows.\looseness=-1
\begin{itemize}
    \item $x^+.\vp = \vg_i$
    \item $x^+.\vDelta = x.\vDelta$
    \item $x^+.t = x.t + \max(\tau_i' - x.t, \frac{\normtwo{\vg_i - x.\vp}}{\tilde{\gamma}^1_i})$
    \item $x^+.\mO$, $x^+.\mD$, and $x^+.\mH$ are computed in the same way as FORWARD.
\end{itemize}
\end{itemize}
Note that we run interaction models only when a robot applies a time-changing action (FORWARD or REACHGOAL), which is an approximation of reality because dynamic objects can potentially change their velocities between robot actions.
We also conduct \emph{conservative collision checks} against dynamic obstacles because we do not include the time domain in the collision check.
This conservatism allows us to preserve collision probability upper bounds against dynamic obstacles during trajectory optimization as discussed in Sec.~\ref{Section:TrajectoryOptimization}.\looseness=-1

We compute the probability of not colliding with static obstacles and a lower bound on the probability of not colliding with dynamic obstacles for each state of the search tree recursively.
We interleave the computation of sets $x.\mO$ and $x.\mD$ with the probability computation.\looseness=-1

\subsubsection{Computing the Probability of Not Colliding with Static Obstacles} 

Let $x^{1:n} = x^1, \ldots, x^n$ be a state sequence in the search tree.
Let $\mC_{s}(x^{l:m})$ be the proposition that the robot following timed path $(x^l.\vp, x^l.t), \ldots, (x^m.\vp, x^m.t)$ collides with any of the static obstacles in $\mO_i$.
The event of not colliding with any of the static obstacles while following a prefix of $x^{1:n}$ admits a recursive definition: $\neg\ \mC_{s}(x^{1:l}) = \neg\ \mC_{s}(x^{1:m}) \bigwedge \neg\ \mC_{s}(x^{m:l})\ \forall l\in\{1,\ldots,n\}\ \forall m\in \{1, \ldots, l\}$.\looseness=-1

We compute the probability $p(\neg\ \mC_{s}(x^{1:l}))$ of not colliding with any of the static obstacles for each prefix $x^{1:l}$ of $x^{1:n}$ during search and store it as metadata of each state.
$p(\neg\ \mC_{s}(x^{1:l}))$ is given by:\looseness=-1
\begin{align*}
    p(&\neg\ \mC_{s}(x^{1:l})) = p(\neg\ \mC_{s}(x^{1:l-1}) \wedge \neg\ \mC_{s}(x^{l-1:l}))\\
    &= p(\neg\ \mC_{s}(x^{1:l-1})) p(\neg\ \mC_{s}(x^{l-1:l})\ |\ \neg\ \mC_{s}(x^{1:l-1}))
\end{align*}

The first term $p(\neg\ \mC_{s}(x^{1:l-1}))$ is the recursive term that can be obtained from the parent state during search.\looseness=-1

The second term $p(\neg\ \mC_{s}(x^{l-1:l}) | \neg\ \mC_{s}(x^{1:l-1}))$ is the conditional term that we compute during state expansion.
Let $\mO^{l:m}_i \subseteq \mO_i$ be the set of static obstacles that collide with robot $i$ traversing $x^{l:m}$.
Given that the robot has not collided while traversing $x^{1:l-1}$ means that no static obstacle that collides with the robot traversing $x^{1:l-1}$ exists.
Therefore, we compute the conditional probability as the probability that none of the obstacles in $\mO^{l-1:l}_i \setminus \mO^{1:l-1}_i$ exists as ones in $\mO^{l-1:l}_i \cap \mO^{1:l-1}_i$ do not exist as presumption.
Let $E(j)$ be the event that static obstacle $j\in\mO_i$ exists.
Assuming independent non-existence events, we have\looseness=-1
\begin{align*}
p(&\neg\ \mC_{s}(x^{l-1:l}) \ |\  \neg\ \mC_{s}(x^{1:l-1})) = p\left(\bigwedge_{j \in \mO^{l-1:l}_i \setminus \mO^{1:l-1}_i} \neg\ E(j)\right)\\
&= \prod_{j \in \mO^{l-1:l}_i \setminus \mO^{1:l-1}_i} p(\neg E(j)) = \prod_{j \in \mO^{l-1:l}_i \setminus \mO^{1:l-1}_i} (1-p_{i,j}^{stat})
\end{align*}

The key operation for computing  the conditional is computing the set $\mO^{l-1:l}_i \setminus \mO^{1:l-1}_i$.
We obtain $\mO^{1:l-1}_i$ from the parent state's $x^{l-1}.\mO$, by using the fact that $x^l.\mO = \mO^{1:l}_i$ by definition for all $l$.
During node expansion, we compute $\mO^{l-1:l}_i$ by querying the static obstacles for collisions against the region swept by $\mR_i^{robot}$ from position $x^{l-1}.\vp$ to $x^l.\vp$.
The probability of not colliding is computed according to obstacles in $\mO^{l-1:l}_i\setminus \mO^{1:l-1}_i$.\looseness=-1

The recursive term $p(\neg\ \mC_{s}(x^{1:1}))$ is initialized for the start state $x^1$ by computing the non-existence probability of obstacles in $x^{1}.\mO$, i.e., $p(\neg\ \mC_{s}(x^{1:1})) = \prod_{j\in x^1.\mO}(1-p_{i,j}^{stat})$.\looseness=-1

\subsubsection{Computing a Lower Bound on the Probability of not Colliding with Dynamic Obstacles}
Let $C_d(x^{l:m})$ be the proposition, conditioned on the full state sequence $x^{1:n}$, that the robot following the $(x^l.\vp, x^l.t), \ldots, (x^m.\vp, x^m.t)$ portion of $x^{1:n}$ collides with any of the dynamic obstacles in $\mD_i$.
Similar to the static obstacles, the event of not colliding with any of the dynamic obstacles while following a prefix of the path $x^{1:n}$ is recursive: $\neg\ C_d(x^{1:l}) = \neg\ C_d(x^{1:m}) \bigwedge \neg\ C_d(x^{m:l})\ \forall l\in\{1,\ldots,n\}\ \forall m\in \{1, \ldots, l\}$.\looseness=-1

The formulation of the probability of not colliding with dynamic obstacles is identical to that developed for static obstacles:
\begin{align*}
    p(&\neg\ \mC_d(x^{1:l})) = p(\neg\ \mC_d(x^{1:l-1}) \wedge \neg\ \mC_d(x^{l-1: l})) \\
    &= p(\neg\ \mC_d(x^{1:l-1}))p(\neg\ \mC_d(x^{l-1:l})\ |\ \neg\ \mC_d(x^{1:l-1}))
\end{align*}

The first term $p(\neg\ \mC_d(x^{1:l-1}))$ is the recursive term that can be can be obtained from the parent state during search.\looseness=-1

The conditional term $p(\neg\ \mC_d(x^{l-1:l})\ |\ \neg\ \mC_d(x^{1:l-1}))$ is computed during state expansion.
Let $C_{d, j}(x^{l:m})$ be the proposition, conditioned on the full state sequence $x^{1:n}$, that the robot following the $(x^l.\vp, x^l.t), \ldots, (x^m.\vp, x^m.t)$ portion of $x^{1:n}$ collides with dynamic obstacle $j\in\mD_i$.
We assume independence between not colliding with different dynamic obstacles; hence, the conditional term simplifies as follows.\looseness=-1
\begin{align*}
    &p(\neg\ \mC_d(x^{l-1:l})\ |\ \neg\ \mC_d(x^{1:l-1})) \\
    &= p\left(\bigwedge_{j \in \mD_i} \neg \ \mC_{d, j}(x^{l-1:l})\ |\ \bigwedge_{j\in\mD_i} \neg\ \mC_{d, j}(x^{1:l-1})\right)\\
    &= \prod_{j \in \mD_i} p(\neg\ C_{d, j}(x^{l-1:l}))\ |\ \neg\ C_{d, j}(x^{1:l-1}))
\end{align*}

The computation of the terms $p(\neg\ C_{d, j}(x^{l-1:l})\ |\ \neg\ C_{d, j}(x^{1:l-1}))$ for each obstacle $j\in\mD_i$ is done by using $x^{l-1}.\mD$ and $x^{l}.\mD$.
Given that robot following states $x^{1:l-1}$ has not collided with dynamic obstacle $j$ means that no behavior model of $j$ that resulted in a collision while traversing $x^{1:l-1}$ is realized.
We store all not colliding dynamic obstacles behavior models in $x^{l-1}.\mD$.
Within these, all dynamic obstacle modes that do not collide with the robot while traversing from $x^{l-1}$ to $x^{l}$ are stored in $x^{l}.\mD$.
Let $x^l.\mD_j$ be the set of all behavior model indices of dynamic obstacle $j\in\mD_i$ that has not collided with $x^{1:l}$.
The probability that the robot does not collide with dynamic obstacle $j$ while traversing from $x^{l-1}$ to $x^{l}$ given that it has not collided with it while traveling from $x^1$ to $x^{l-1}$ is given by\looseness=-1
\begin{align*}
    p(\neg\ C_{d, j}(x^{l-1:l})\ |\ \neg\ C_{d,j}(x^{1:l-1})) &=\frac{\sum_{k \in x^l.\mD_j} p_{i,j,k}^{dyn}}{\sum_{k \in x^{l-1}.\mD_j}p_{i,j,k}^{dyn}}.
\end{align*}

The computed probabilities for not colliding are lower bounds because \emph{collision checks against dynamic obstacles are done conservatively}, i.e., the time domain is not considered during sweep to sweep collision checks.
Conservative collision checks never miss collisions but may over-report them.\looseness=-1

\textbf{Costs.} 
Let $p_{s}(x^l) = 1-p(\neg\ C_s(x^{1:l}))$ be the probability of collision with any of the static obstacles and $p_{d}(x^l) = 1-p(\neg\ C_d(x^{1:l}))$ be an upper bound for the probability of collision with any of the dynamic obstacles while traversing state sequence $x^{1:l}$.
We define $P_{s}(t): [0, x^n.t] \rightarrow [0,1]$ of state sequence $x^{1:n}$ as the linear interpolation of $p_{s}$:
\begin{align*}
    P_{s}(t) &= 
    \begin{cases}
        \frac{x^2.t - t}{x^2.t-x^1.t}p_{s}(x^1) \\\ \ \ +\frac{t-x^1.t}{x^2.t-x^1.t}p_{s}(x^2) & x^1.t\leq t <x^2.t\\
        \ldots\ \\
        \frac{x^n.t - t}{x^n.t-x^{n-1}.t}p_{s}(x^{n-1}) \\\ \ \ +\frac{t-x^{n-1}.t}{x^n.t-x^{n-1}.t}p_{s}(x^n) & x^{n-1}.t\leq t \leq x^n.t
        \end{cases}
\end{align*}
We define $P_{d}(t): [0, x^n.t] \rightarrow [0,1]$ of a state sequence $x^{1:n}$ in a similar way using $p_{d}$.
We define $P_c(t):[0, x^n.t]\rightarrow[0, \infty)$ as the linear interpolation of the number of violated hyperplanes in active DSHTs of the state sequence $x^{1:n}$, i.e., the points $(x^1.t, |x^1.\mH|), \ldots, (x^n.t, |x^n.\mH|)$.\looseness=-1


We associate six different cost terms to each state $x^l$ in state sequence $x^{1:n}$: i) $\mJ_{static}(x^l) \in [0, \infty)$ is the cumulative static obstacle collision probability defined as $\mJ_{static}(x^l) = \int_0^{x^l.t}P_s(\tau)d\tau$, ii) $\mJ_{dynamic}(x^l)\in [0, \infty)$ is the cumulative dynamic obstacle collision probability defined as $\mJ_{dynamic}(x^l) = \int_0^{x^l.t}P_d(\tau)d\tau$, iii) $\mJ_{team}(x^l) \in [0, \infty)$ is the cumulative number of violated active DSHT hyperplanes defined as $\mJ_{team}(x^l) = \int_0^{\min(x^l.t, T^{team}_i)}P_c(\tau)d\tau$, in which violation cost accumulation is cut off at $T_i^{team}$ parameter, iv) $\mJ_{distance}(x^l) \in [0, \infty)$ is the distance traveled from start state $x^1$ to state $x^l$, v) $\mJ_{duration}(x^l) \in [0, \infty)$ is the time elapsed from start state $x^1$ to state $x^l$, and vi) $\mJ_{rotation}(x^l) \in \mathbb{N}$ is the number of rotations from start state $x^1$ to state $x^l$.\looseness=-1

We cut off violation cost accumulation of DSHTs because of the conservative nature of using separating hyperplanes for teammate safety: they divide the space into two disjoint sets linearly without considering the robots' intents.
The robots need to be safe until the next successful planning iteration because of the receding horizon planning, and overly constraining a large portion of the plan at each planning iteration with conservative constraints decreases agility.
We investigate the effects of $T^{team}_i$ on navigation performance in Sec.~\ref{Section:TeammateSafetyEnforcementDurationEvaluation}.\looseness=-1

We compute the cost terms of the new state $x^+$ after applying actions to the current state $x$ as follows.
\begin{itemize}
    \item $\mJ_{static}(x^+) = \mJ_{static}(x) + \int_{x.t}^{x^+.t}P_s(\tau)\tau$
    \item $\mJ_{dynamic}(x^+) = \mJ_{dynamic}(x) + \int_{x.t}^{x^+.t}P_d(\tau)\tau$
    \item $\mJ_{team}(x^+) = \mJ_{team}(x) + \int_{\min(x.t, T_i^{team})}^{\min(x^+.t, T_i^{team})}P_c(\tau)d\tau$
    \item $\mJ_{distance}(x^+) = \mJ_{distance}(x) + \normtwo{x^+.\vp - x.\vp}$
    \item $\mJ_{duration}(x^+) = \mJ_{duration}(x) + (x^+.t - x.t)$
    \item $\mJ_{rotation}(x^+) = \mJ_{rotation}(x) + \mathbbm{1}_{\neq}(x.\vDelta, x^+.\vDelta)$
\end{itemize}
where $\mathbbm{1}_{\neq}$ is the indicator function with value $1$ if its arguments are unequal, and $0$ otherwise.\looseness=-1

Lower cost (with respect to standard comparison operator $\leq$) is better in all cost terms.
All cost terms have the minimum of $0$ and upper bound of $\infty$.
All cost terms are additive using the standard addition.
$\mJ_{static}, \mJ_{dynamic}, \mJ_{team}, \mJ_{distance}, $ and $\mJ_{duration}$ are cost algebras ($\mathbb{R}_{\geq 0} \cup \{\infty\}$, min, $+$, $\leq$, $\infty$, $0$) and $\mJ_{rotation}$ is cost algebra ($\mathbb{N} \cup \{\infty\}$, min, $+$, $\leq$, $\infty$, $0$), both of which are strictly isotone.
Therefore, their Cartesian product is also a cost algebra, which is what we optimize over.
The cost $\mJ(x)$ of each state $x$ is:\looseness=-1
\begin{align*}
    \mJ(x) = \begin{bmatrix}\mJ_{static}(x)\\ \mJ_{dynamic}(x)\\ \mJ_{team}(x) \\\mJ_{distance}(x)\\ \mJ_{duration}(x)\\ \mJ_{rotation}(x)\end{bmatrix}.
\end{align*} 
We order cost terms lexicographically.
A sample state sequence and computed costs are shown in Fig.~\ref{Figure:DiscreteSearch}.\looseness=-1

This induces an ordering between cost terms: we first minimize cumulative static obstacle collision probability, and among the states that minimize that, we minimize cumulative dynamic obstacle collision probability, and so on.\footnote{Note that we do not explicitly find such solutions; but this behavior is naturally provided by the cost algebraic A* as it finds the optimal plan according to the lexicograpical ordering of the cost terms.\looseness=-1}
Hence, safety is the primary; distance, duration, and rotation optimality are the secondary concerns.
Out of safety with respect to static and dynamic obstacles and teammates, we prioritize static obstacles over dynamic obstacles, because static obstacles can be considered a special type of dynamic ones, i.e., with $\vzero$ velocity, and hence, prioritizing dynamic obstacles would make the static obstacle avoidance cost unnecessary.
This ordering allows us to optimize the special case first, and then attempt the harder one.
The reason we prioritize dynamic obstacles over teammates is the conservative nature of using DSHTs for teammates.
Violating a separating hyperplane does not necessarily result in a collision because each hyperplane divides the space into two between robots, and the robots occupy a very small portion of their side of each hyperplane.\looseness=-1

The heuristic $H(x)$ we use for each state $x$ during search is as follows.
\resizebox{\linewidth}{!}{
\begin{minipage}{\linewidth}
\begin{align*}
H(x) &= \begin{bmatrix}
H_{static}(x) \\
H_{dynamic}(x)\\
H_{team}(x)\\
H_{distance}(x)\\
H_{duration}(x)\\
H_{rotation}(x)
\end{bmatrix} 
= \begin{bmatrix}
P_s(x.t)H_{duration}(x)\\
P_d(x.t)H_{duration}(x)\\
P_c(x.t)\max(0, \min(H_{duration}(x), T_i^{team} - x.t))\\
\normtwo{x.\vp - \vg_i}\\
\max(\tau_i'-x.t, \frac{H_{distance}(x)}{\tilde{\gamma_1}})\\
0
\end{bmatrix}
\end{align*}
\end{minipage}
}
\vspace{0.05in}

We first compute $H_{distance}(x)$, which we use in the computation of $H_{duration}(x)$.
Then, we use $H_{duration}(x)$ during the computation of $H_{static}(x)$, $H_{dynamic}(x)$, and $H_{team}(x)$.\looseness=-1

\begin{proposition}
    All individual heuristics are admissible.
\end{proposition}
\begin{proof}
    
\textbf{Admissibility of $\boldsymbol{H_{distance}}$:} $H_{distance}(x)$ is the Euclidean distance from $x.\vp$ to $\vg_i$, and never overestimates the true distance.\looseness=-1

\textbf{Admissibility of $\boldsymbol{H_{duration}}$:} The goal position $\vg_i$ can be any position in $\mathbb{R}^d$, which is an uncountable set.
The FORWARD and ROTATE actions can only move the robot to a discrete set of positions, which is countable, as any discrete subset of a Euclidean space is countable. 
Therefore, the probability that the robot reaches $\vg_i$ by only executing FORWARD and ROTATE actions is zero.
The robot cannot execute any action after REACHGOAL action in an optimal path to a goal state, because the REACHGOAL action already ends in the goal position and any subsequent actions would only increase the total cost.
Hence, the last action in an optimal path to a goal state should be REACHGOAL.
There are two cases to consider.\looseness=-1

If the last action while arriving at state $x$ is REACHGOAL, $x.t \geq \tau_i'$ holds (as REACHGOAL enforces this, see the descriptions of actions). Since $x.\vp = \vg_i$, $H_{distance}(x) = 0$. Therefore, $H_{duration}(x) = \max(\tau_i' -x.t, \frac{H_{distance}(x)}{\tilde{\gamma}_i^1}) = 0$, which is trivially admissible, as $0$ is the lowest cost.\looseness=-1

If the last action while arriving at state $x$ is not REACHGOAL, the search should execute REACHGOAL action to reach to the goal position in the future, which enforces that goal position will not be reached before $\tau_i'$. 
Also, since the maximum speed that can be executed during search is $\tilde{\gamma}_i^1$, robot needs at least $\frac{H_{distance}(x)}{\tilde{\gamma}_i^1}$ seconds to reach to the goal position as $H_{distance}(x)$ is an admissible heuristic for distance to goal position.
Hence, $H_{duration}(x) = \max(\tau_i' - x.t, \frac{H_{distance}(x)}{\tilde{\gamma}_i^1})$ is admissible.\looseness=-1

\textbf{Admissibility of $\boldsymbol{H_{static}}$ and $\boldsymbol{H_{dynamic}}$:}
We prove the admissibility of $H_{static}$.
Proof of admissibility of $H_{dynamic}$ follows identical steps.
$P_s$ is a nondecreasing nonnegative function as it is the accumulation of linear interpolation of probabilities, which are defined over $[0, 1]$.
Therefore, $P_s(x.t) \leq P_s(t)$ for $t \geq x.t$ in an optimal path to a goal state traversing $x$.
The robot needs at least $H_{duration}(x)$ seconds to reach a goal state from $x$, since $H_{duration}$ is an admissible heuristic.
Let $T^g(x) \geq H_{duration}(x)$ be the actual duration needed to reach to a goal state from $x$ on an optimal path.
The actual cumulative static obstacle collision probability to a goal on an optimal path from $x$ is $\int_{x.t}^{x.t + T^g(x)}P_s(\tau)d\tau$.
We have
\begin{align*}
    H_{static}(x) &= P_s(x.t)H_{duration}(x) \\
    &= \int_{x.t}^{x.t+H_{duration}(x)}P_s(x.t)d\tau \\
    &\leq \int_{x.t}^{x.t + T^g(x)}P_s(x.t)d\tau \leq \int_{x.t}^{x.t + T^g(x)}P_s(\tau)d\tau.
\end{align*}
In other words, $H_{static}(x)$ does not overestimate the true static obstacle cumulative collision probability from $x$ to a goal state.\looseness=-1

\textbf{Admissibility of $\boldsymbol{H_{team}}$:} 
Let $x^{1:n}$ be an optimal state sequence from start state $x^1$ to a goal state $x^n$ traversing $x$.
If $x.t \geq T_i^{team}$, $P_c(t)$ will not be accumulated in the future because of the cut-off.
If $x.t \leq T_i^{team}$, $P_c(t)$ will be accumulated for a duration at least $\min(H_{duration}(x), T_i^{team} - x.t)$ because $H_{duration}$ never overestimates the true duration to a goal and accumulation is cut off at $T_i^{team}$.
Therefore, $P_c(t)$ will be accumulated for at least $\max(0, \min(H_{duration}(x), T_i^{team} - x.t))$ after state $x$.
Let $T^c(x) \geq \max(0, \min(H_{duration}(x), T_i^{team} - x.t))$ be the actual duration $P_c(t)$ will be accumulated after state $x$.
The actual cumulative number of violated active DSHT hyperplanes is given by $\int_{x.t}^{x.t+T^c(x)}P_c(\tau)d\tau$.\looseness=-1

$|x^l.\mH|\geq|x^{l-1}.\mH|$ for all $l\in\{2, \ldots, n\}$ because if a hyperplane is violated while traversing $x^1,\ldots, x^{l-1}$, it is also violated while traversing $x^1, \ldots, x^l$.
Therefore, linear interpolation $P_c(t)$ of the number of violated hyperplanes is a nondecreasing function, i.e., $P_c(x.t) \leq P_c(t)\ \forall t\in[x.t, x^n.t]$.
In addition, $P_c(t)$ is a nonnegative function as it is a linear interpolation of set cardinalities.
Hence, we have
\begin{align*}
    H_{team}&(x) = P_c(x.t) \times\\
    &\ \ \ \ \ \ \ \ \max(0, \min(H_{duration}(x), T_i^{team} - x.t))\\
    &=\int_{x.t}^{x.t + \max(0, \min(H_{duration}(x), T_i^{team} - x.t))}P_c(x.t)d\tau\\
    &\leq \int_{x.t}^{x.t+T^c(x)}P_c(x.t)d\tau \leq \int_{x.t}^{x.t+T^c(x)}P_c(\tau)d\tau.
\end{align*}
In other words, $H_{team}(x)$ never overestimates true accumulated $P_c(t)$ in an optimal path to the goal state $x^n$ from $x$.\looseness=-1

\textbf{Admissibility of $\boldsymbol{H_{rotation}}$:} $H_{rotation} = 0$ is trivially admissible because $0$ is the lowest cost.

\end{proof}

As each individual cost term is admissible, their Cartesian product is also admissible.
Hence, cost algebraic A* with re-openings minimizes $\mJ$ with the given heuristics $H$.\looseness=-1


\textbf{Time limited best effort search.} 
Finding the optimal state sequence induced by the costs $\mJ$ to the goal position $\vg_i$ can take a lot of time.
Because of this, each robot $i$ limits the duration of the search using a maximum search duration parameter $T^{search}_i$.
When the search gets cut off because of the time limit, we return the lowest cost state sequence to a goal state so far.
During node expansion, A* applies all actions to the state with the lowest cost.
One of those actions is always REACHGOAL.
Therefore, we connect all expanded states to the goal position using REACHGOAL action.
Hence, when the search is cut off, there are many candidate plans to the goal position, which are already sorted according to their costs by A*.\looseness=-1

We remove the states generated by ROTATE actions from the search result and provide the resulting sequence to the trajectory optimization stage.
Note that ROTATE changes only the direction $x.\vDelta$.
The trajectory optimization stage does not use $x.\vDelta$, therefore we remove repeated states for the input of trajectory optimization.\looseness=-1

\subsection{Trajectory Optimization}\label{Section:TrajectoryOptimization}

Let $x^1,\ldots, x^N$ be the state sequence provided to trajectory optimization.
In the trajectory optimization stage, each robot $i$ fits a B\'ezier curve $\vf_{i,l}(t): [0, T_{i,l}] \rightarrow \mathbb{R}^d$ of degree $h_{i,l}$ (which are parameters) where $T_{i,l} = x^{l+1}.t - x^l.t$ to each segment from $(x^l.t, x^l.\vp)$ to $(x^{l+1}.t, x^{l+1}.\vp)$ for all $l \in \{1, \ldots, N-1\}$ to compute a piecewise trajectory $\vf_i(t):[0, x^N.t]\rightarrow \mathbb{R}^d$ where each piece is the fitted B\'ezier curve, i.e.
\begin{align*}
    \vf_i(t) = \begin{cases}
        \vf_{i,1}(t - x^1.t) & x^1.t = 0 \leq t < x^2.t\\
        \ldots&\ \\
        \vf_{i,N-1}(t - x^{N-1}.t) & x^{N-1}.t\leq t \leq x^{N}.t
    \end{cases}
\end{align*}
A B\'ezier curve $\vf_{i,l}(t)$ of degree $h_{i,l}$ has $h_{i,l}+1$ control points $\vP_{i, l, 0}, \ldots, \vP_{i, l, h_{i,l}}$ and is defined as
\begin{align*}
    \vf_{i,l}(t) = \sum_{k = 0}^{h_{i,l}} \vP_{i, l, k}\binom{h_{i,l}}{k}\left(\frac{t}{T_{i,l}}\right)^k\left(1-\frac{t}{T_{i,l}}\right)^{h_{i,l} - k}
\end{align*}

B\'ezier curves are contained in the convex hull of their control points~\cite{farouki2012bernstein}.
This allows us to easily constrain B\'ezier curves to be inside convex sets by constraining their control points to be in the same convex sets.
Let $\mP_i = \{\vP_{i, l, k}\ |\ l\in\{1, \ldots, N-1\}, k\in \{0, \ldots, h_{i,l}\} \}$ be the set of the control points of all pieces of robot $i$.\looseness=-1

During the trajectory optimization stage, we construct a QP where decision variables are control points $\mP_i$ of the trajectory.\looseness=-1

We preserve the cumulative collision probabilities $P_{s}$ and $P_{d}$ and the cumulative number of violated active DSHT hyperplanes $P_{c}$ of the state sequence $x^{1:N}$, by ensuring that robot following $\vf_{i,l}(t)$ i) avoids the same static obstacles and dynamic obstacle behavior models robot traveling from $x^1$ to $x^{l+1}$ avoids and ii) does not violate any active DSHT hyperplane that the robot traveling from $x^{1}$ to $x^{l+1}$ does also not violate for all $l\in\{1,\ldots,N-1\}$.
The constraints generated for these are always feasible.\looseness=-1

In order to encourage dynamic obstacles to determine their behavior using the interaction models in the same way they determine it in response to $x^{1:N}$, we add cost terms that match the position and the velocity at the start of each B\'ezier piece $\vf_{i,l}(t)$ to the position and the velocity of the robot at $x^{l}$ for each $l\in\{1, \ldots, N-1\}$.\looseness=-1

\subsubsection{Constraints}
There are five types of constraints we impose on the trajectory, all of which are linear in control points $\mP_i$.


\begin{figure}
    \centering
    \includegraphics[width=\linewidth]{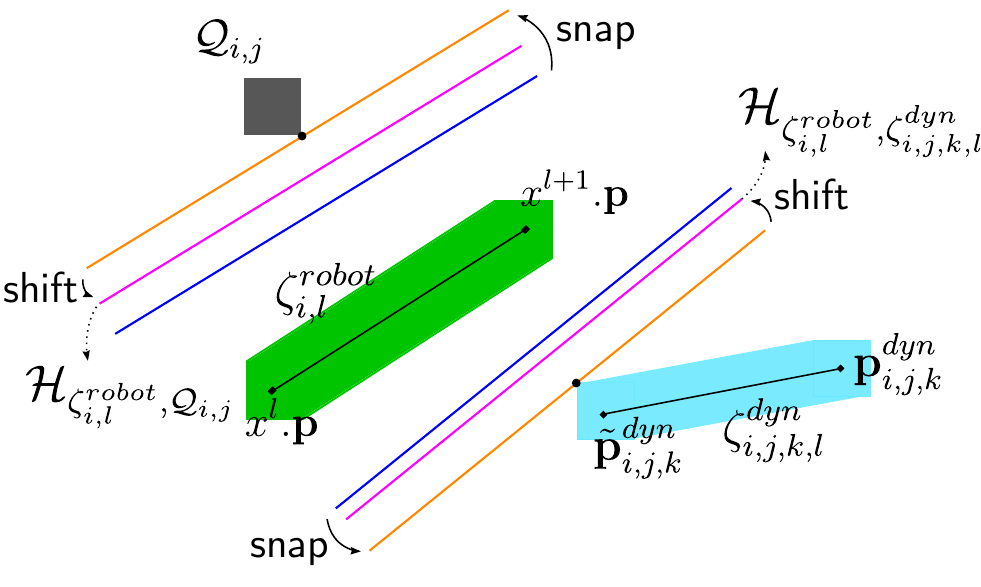}
    \vspace{-0.3in}
    \caption{\textbf{Static and dynamic obstacle collision constraints.} Given the {\color{darkgray}gray} static obstacle $j\in\mO_i$ with shape $\mQ_{i,j}$ and the {\color{ForestGreen} green} sweep $\zeta_{i,l}^{robot}$ of $\mR_i^{robot}$ from $x^l.\vp$ to $x^{l+1}.\vp$, we compute the {\color{blue}blue} support vector machine hyperplane between them. We compute the {\color{orange}orange} separating hyperplane by snapping it to $\mQ_{i,j}$. The robot should stay in the safe side of the {\color{orange}orange} hyperplane. We shift {\color{orange}orange} hyperplane to account for robot's collision shape $\mR_i^{robot}$ and compute the {\color{magenta}magenta} hyperplane $\mH_{\zeta_{i,l}^{robot},\mQ_{i,j}}$. The B\'ezier curve $\vf_{i,l}(t)$ is constrained by $\mH_{\zeta_{i,l}^{robot},\mQ_{i,j}}$ to avoid $\mQ_{i,j}$. To avoid the dynamic obstacle $j\in \mD_i$ moving from $\tilde{\vp}^{dyn}_{i, j, k}$ to $\vp^{dyn}_{i, j, k}$, the {\color{blue}} support vector machine hyperplane between the region $\zeta^{dyn}_{i,j,k,l}$ swept by $\mR_{i,j}^{dyn}$ and $\zeta_{i,l}^{robot}$ is computed. The same snap and shift operations are conducted to compute the {\color{magenta} magenta} hyperplane $\mH_{\zeta_{i,l}^{robot},\zeta^{dyn}_{i,j,k,l}}$, constraining $\vf_{i,l}(t)$.}
    \label{Figure:SeparatingHyperplanes}
    \vspace{-0.3in}
\end{figure}

\textbf{Static obstacle avoidance constraints.}
Let $j \in \mO_i \setminus x^{l+1}.\mO$ be a static obstacle that robot $i$ travelling from $x^1$ to $x^{l+1}$ avoids for an $l\in\{1, \ldots, N-1\}$.
Let $\zeta_{i,l}^{robot}$ be the space swept by the robot traveling the straight line from $x^l.\vp$ to $x^{l+1}.\vp$.
Since the shape of the robot is convex and it is swept along a straight line segment, $\zeta_{i,l}^{robot}$ is also convex~\cite{senbaslar2023rlss}.
Static obstacle $j$ is also convex by definition.
Since robot avoids $j$, $\mQ_{i,j} \cap \zeta_{i,l}^{robot} = \emptyset$.
Hence, they are linearly separable by the separating hyperplane theorem.
We compute the support vector machine (SVM) hyperplane between $\zeta_{i,l}^{robot}$ and $\mQ_{i,j}$, snap it to $\mQ_{i,j}$ by shifting it along its normal so that it touches $\mQ_{i,j}$, and shift it back to account for robot's collision shape $\mR_i^{robot}$ similarly to~\cite{senbaslar2023rlss} (Fig.~\ref{Figure:SeparatingHyperplanes}).
Let $\mH_{\zeta_{i,l}^{robot}, \mQ_{i,j}}$ be this hyperplane.
We constrain $\vf_{i,l}$ with $\mH_{\zeta_{i,l}^{robot}, \mQ_{i,j}}$ for it to avoid static obstacle $j$, which is a feasible linear constraint as shown in~\cite{senbaslar2023rlss}.\looseness=-1

These constraints enforce that robot traversing $\vf_{i,l}(t)$ avoids the same obstacles robot traversing from $x^1$ to $x^{l+1}$ avoids, not growing the set $x^{l+1}.\mO$ between $[x^l.t, x^{l+1}.t]\ \forall l\in\{1, \ldots, N-1\}$, and hence preserving $P_{s}(t)$ $\forall t\in[0, x^N.t]$.\looseness=-1

\textbf{Dynamic obstacle avoidance constraints.}
Let $(\mB_{i,j,k}, \vp_{i,j,k}^{dyn}) \in x^{l+1}.\mD$ be a dynamic obstacle behavior model--position pair that does not collide with robot travelling from $x^{1}$ to $x^{l+1}$ for an $l\in\{1, \ldots, N-1\}$.
$\mB_{i, j, k}$ should be in $x^l.\mD$ as well, because the behavior models in $x^{l+1}.\mD$ are a subset of behavior models in $x^l.\mD$ by definition.
Let $\tilde{\vp}_{i, j, k}^{dyn}$ be the position of the dynamic obstacle $j$ moving according to behavior model $\mB_{i,j,k}$ at state $x^l$.
Let $\zeta^{dyn}_{i, j, k, l}$ be the region swept by the dynamic object $j$ from  $\tilde{\vp}_{i, j, k}^{dyn}$ to $ \vp_{i,j,k}^{dyn}$.
During collision check of state expansion from $x^l$ to $x^{l+1}$, we check whether $\zeta^{dyn}_{i, j, k, l}$ intersects with $\zeta_{i,l}^{robot}$ and add the model to $x^{l+1}.\mD$ if they do not.
Since these sweeps are convex sets (because they are sweeps of convex sets along straight line segments), they are linearly separable.
We compute the SVM hyperplane between them, snap it to the region swept by dynamic obstacle and shift it back to account for the robot shape $\mR_i^{robot}$ (Fig.~\ref{Figure:SeparatingHyperplanes}).
Let $\mH_{\zeta_{i,l}^{robot}, \zeta^{dyn}_{i, j, k, l}}$ be this hyperplane.
We constrain $\vf_{i,l}$ with $\mH_{\zeta_{i,l}^{robot}, \zeta^{dyn}_{i, j, k, l}}$, which is a feasible linear constraint as shown in~\cite{senbaslar2023rlss}.\looseness=-1

These constraints enforce that robot traversing $\vf_{i,l}(t)$ avoids same dynamic obstacle behavior models robot travelling from $x^1$ to $x^{l+1}$ avoids, not shrinking the set $x^{l+1}.\mD\ \forall l\in\{1,\ldots, N-1\}$, and hence preserving $P_{d}(t)$ $\forall t\in[0, x^N.t]$.\looseness=-1

The reason we perform \emph{conservative collision checks} for dynamic obstacle avoidance during discrete search is to use the separating hyperplane theorem.
Without the conservative collision check, there is no proof of linear separability, and SVM computation might fail.\looseness=-1

\textbf{Teammate avoidance constraints.}
Let $\mH \in \tilde{\mH}^{active}_i\setminus x^{l+1}.\mH$ be an active DSHT hyperplane that is not violated while traversing states from $x^1$ to $x^{l+1}$.
If $x^l.t < T_i^{team}$, i.e., the segment from $x^l$ to $x^{l+1}$ is within the teammate safety enforcement period, we constrain $\vf_{i, l}$ with $\mH$ by shifting it to account for robot's collision shape, and enforcing $\vf_{i, l}$ to be in the safe side of the shifted hyperplane, which is a feasible constraint~\cite{senbaslar2023rlss}.
Otherwise, we do not constrain the piece $\vf_{i, l}$ with active DSHT hyperplanes.\looseness=-1

Within the safety enforcement period $T_i^{team}$, any $\vf_{i, l}$ does not violate any active DSHT hyperplane that is not violated while traversing the state sequence $x^{1:l+1}$, preserving the cardinality of sets $x^{l+1}.\mH$, and hence $P_c(t)$.\looseness=-1

\textbf{Continuity constraints.}
We enforce continuity up to the desired degree $c_i$ between pieces by
\begin{align*}
    \frac{d^k\vf_{i,l}(T_{i,l})}{dt^k} = \frac{d^k\vf_{i,l+1}(0)}{dt^k}\ &\forall l \in \{1,\ldots, N-2\}\\
    &\forall k\in\{0,\ldots,c_i\}.
\end{align*}

We enforce continuity up to desired degree $c_i$ between planning iterations by
\begin{align*}
    \frac{d^k\vf_{i}(0)}{dt^k} = \vp_{i,k}^{self}\ \forall k\in\{0, \ldots, c_i\}.
\end{align*}

\textbf{Dynamic limit constraints.}
Derivative of a B\'ezier curve is another B\'ezier curve with a lower degree, control points of which are linearly related to the control points of the original curve~\cite{farouki2012bernstein}.
Let $\mP_i^k = \vD^k(\mP_i)$ be the control points of the $k^{th}$ derivative of $\vf_i$, where $\vD^k$ is the linear transformation relating $\mP_i$ to $\mP_i^k$.
We enforce dynamic constraints uncoupled among dimensions by limiting maximum $k^{th}$ derivative magnitude in each dimension by $\frac{\gamma_i^k}{\sqrt{d}}$ so that they are linear.
Utilizing the convex hull property of B\'ezier curves, we enforce
\begin{align*}
 -\frac{\gamma_i^k}{\sqrt{d}} \preceq\vP \preceq \frac{\gamma_i^k}{\sqrt{d}}\ \forall\vP\in\mP_i^k   
\end{align*}
which limits $k^{th}$ derivative magnitude with $\gamma_i^k$ along the trajectory $\vf_i$ where $\preceq$ between a vector and a scalar is true if and only if all elements of the vector are less than or equal to the scalar.\looseness=-1

While collision avoidance constraints are always feasible, we do not have a general proof of feasibility for continuity and dynamic limit constraints, which may cause the planner to fail.
If the planner fails, the robot continues using the last successfully planned trajectory.\looseness=-1


\subsubsection{Objective Function}
We use a linear combination of three cost terms as our objective function, all of which are quadratic in control points $\mP_i$.\looseness=-1

\textbf{Energy term.}
We use the sum of integrated squared derivative magnitudes as a metric for energy usage similar to~\cite{senbaslar2023rlss, honig2018quadswarms, richter2013planning, senbaslar2018rte}. The energy usage cost term $\mJ_{energy}(\mP_i)$ is
\begin{align*}
\mJ_{energy}(\mP_i) = \sum_{\lambda_{i,k} \in \vlambda_i} \lambda_{i,k} \int_0^{x^N.t}\normtwo{\frac{d^k\vf_i(t)}{dt^k}}^2dt
\end{align*}
where $\lambda_{i,k}$s are parameters.\looseness=-1

\textbf{Position matching term.}
We add a position matching term $J_{position}(\mP_i)$ that penalizes distance between piece endpoints and state sequence positions $x^{2}.\vp, \ldots, x^{N}.\vp$.\looseness=-1

\begin{align*}
\mJ_{position}(\mP_i) = \sum_{l\in\{1,\ldots,N-1\}} \theta_{i,l}\normtwo{\vf_{i,l}(T_{i,l}) - x^{l+1}.\vp}^2
\end{align*}
where $\theta_{i,l}$s are weight parameters.\looseness=-1

\textbf{Velocity matching term.}
We add a velocity matching term $J_{velocity}$ that penalizes divergence from the velocities of the state sequence $x^{1:N}$ at piece start points.\looseness=-1

\begin{align*}
    \mJ_{velocity}(\mP_i) = \sum_{l\in\{1,\ldots,N-1\}} \beta_{i,l}\normtwo{\frac{d\vf_{i,l}(0)}{dt} - \frac{x^{l+1}.\vp - x^l.\vp}{x^{l+1}.t - x^l.t}}^2
\end{align*}
where $\beta_{i,l}$s are weight parameters.\looseness=-1

Position and velocity matching terms encourage matching the positions and velocities of the state sequence $x^{1:N}$. This causes dynamic obstacles to make similar interaction decisions against the robot following trajectory $\vf_i(t)$ to they do to the robot following the state sequence $x^{1:N}$.
One could also add constraints to the optimization problem to exactly match positions and velocities.
Adding position and velocity matching terms as constraints resulted in a high rate of optimization infeasibilities in our experiments.
Therefore, we choose to add them to the cost function of the optimization term in the final algorithm.\looseness=-1

\section{Evaluation}

We implement our algorithm in C++.
We use CPLEX 12.10 to solve the quadratic programs generated during the trajectory optimization stage, including the SVM problems.
We evaluate our planner's behavior in simulations in 3D.
All simulation experiments are conducted on a computer with Intel(R) i7-8700K CPU @3.70GHz, running Ubuntu 20.04 as the operating system.
The planning pipeline is executed in a single core of the CPU in each planning iteration of each robot.
We compare our algorithm's performance with three state-of-the-art decentralized navigation decision-making algorithms, namely SBC~\cite{wang2017safety}, RLSS~\cite{senbaslar2023rlss}, and RMADER~\cite{kondo2023rmader}, in both single-robot and multi-robot scenarios in simulations and show that our algorithm achieves considerably higher success rate compared to the baselines.
We implement our algorithm for physical quadrotors and show its feasibility in the real world in single and multi-robot experiments.
The supplemental video includes recordings from both i) simulations, including some not covered in this paper, and ii) physical robot experiments.\looseness=-1


\subsection{Simulation Evaluation Setup}

\subsubsection{Obstacle Sensing}\label{Section:Mocking}
We use octrees~\cite{hornung2013octomap} to represent the static obstacles.
Each axis-aligned box with its stored existence probability is used as a static obstacle.
We model static obstacle sensing imperfections using three operations applied to the octree representation of the environment in static obstacle sensing uncertainty experiments (Sec.~\ref{Section:StatiObstacleSensingUncertainty}):
\begin{itemize}
\item \textbf{increaseUncertainty:} Increases the uncertainty of existing obstacles by moving their existence probabilities closer to $0.5$, by sampling a probability between the existence probability $p$ of an obstacle and $0.5$ uniformly.
\item \textbf{leakObstacles($\boldsymbol{p_{leak}}$):} Leaks each obstacle to a neighbouring region with probability $p_{leak}$.
\item \textbf{deleteObstacles:} Deletes obstacles randomly according to their non-existence probabilities.
\end{itemize}\looseness=-1

We model dynamic obstacle shapes $\mR_{i,j}^{dyn}$ as axis-aligned boxes.
For each robot, $i$, to simulate imperfect sensing of $\mR_{i,j}^{dyn}$, we inflate or deflate it in each axis randomly according to a one dimensional $0$ mean Gaussian noise with standard deviation $\sigma_i$ in experiments with dynamic obstacle sensing uncertainty (Sec.~\ref{Section:DynamicObstacleSensingUncertainty})\footnote{Note that, we do not explicitly account for the dynamic obstacle shape sensing uncertainty during planning, yet we still show our algorithm's performance under such uncertainty.}.
Each robot $i$ is assumed to be noisily sensing the position and velocity of each dynamic obstacle $j\in \mD_i$ according to a $2d$ dimensional $0$ mean Gaussian noise with positive definite covariance $\Sigma_i \in \mathbb{R}^{2d \times 2d}$.
The first $d$ terms of the noise are applied to the real position and the second $d$ terms of the noise are applied to the real velocity of the obstacle to compute sensed position and velocity at each simulation step.\looseness=-1

\subsubsection{Predicting Behavior Models of Dynamic Obstacles}\label{Section:Prediction}
We introduce three simple model-based online dynamic obstacle behavior model prediction methods to use during evaluation\footnote{More sophisticated behavior prediction methods can be developed and integrated with our planner, which might potentially use domain knowledge about the environment objects exists or handle position and velocity sensing uncertainties explicitly.}.\looseness=-1

\begin{figure}
    \centering
    \hfill
    \subfloat[Goal attractive movement]{\includegraphics[width=0.24\linewidth]{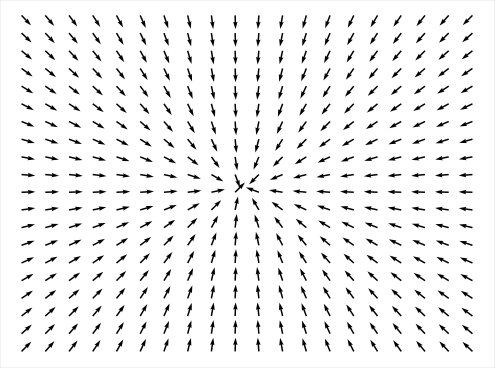}\label{Figure:GoalAttractive}}\hfill
    \subfloat[Constant velocity movement]{\includegraphics[width=0.24\linewidth]{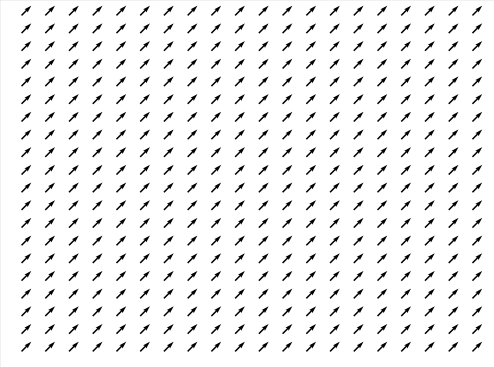}\label{Figure:ConstantVelocity}}\hfill
    \subfloat[Rotating movement]{\includegraphics[width=0.24\linewidth]{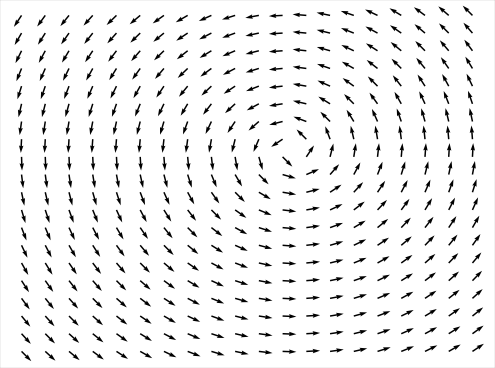}\label{Figure:Rotating}}\hfill
    \subfloat[Repulsive interaction]{\includegraphics[width=0.24\linewidth]{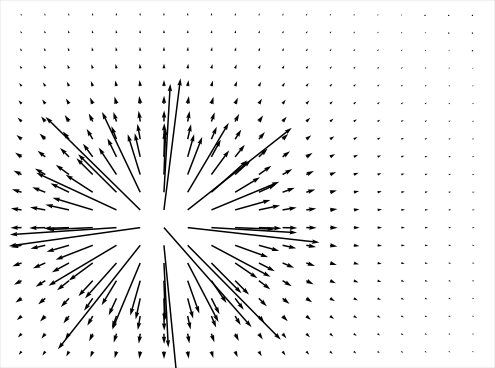}\label{Figure:Repulsive}}\hfill
    \caption{The movement and interaction models we define for dynamic obstacles. Each model has associated parameters described in Sec.~\ref{Section:Prediction}.}
    \vspace{-0.25in}
\end{figure}

Let $\vp^{dyn}$ be the position of a dynamic obstacle. We define three movement models:
\begin{itemize}
\item \textbf{Goal attractive movement model $\boldsymbol{\mM_g(\vp^{dyn} | \hat{\vg}, \hat{s})}$  (Fig.~\ref{Figure:GoalAttractive}):} Attracts the dynamic obstacle to the goal position $\hat{\vg}$ with desired speed $\hat{s}$.
The desired velocity $\tilde{\vv}^{dyn}$ of the dynamic obstacle is computed as $\tilde{\vv}^{dyn} = \mM_g(\vp^{dyn} | \hat{\vg}, \hat{s}) = \frac{\hat{\vg} - \vp^{dyn}}{\normtwo{\hat{\vg} - \vp^{dyn}}}\hat{s}$.
\item \textbf{Constant velocity movement model $\boldsymbol{\mM_c(\vp^{dyn} | \hat{\vv})}$ (Fig.~\ref{Figure:ConstantVelocity}):} Moves the dynamic obstacle with constant velocity $\hat{\vv}$.
The desired velocity $\tilde{\vv}^{dyn}$ of the dynamic obstacle is computed as $\tilde{\vv}^{dyn} = \mM_c(\vp^{dyn} | \hat{\vv}) = \hat{\vv}$.
\item \textbf{Rotating movement model $\boldsymbol{\mM_r(\vp^{dyn} | \hat{\vc}, \hat{s})}$ (Fig.~\ref{Figure:Rotating}):} Rotates the robot around the rotation center $\hat{\vc}$ with desired speed $\hat{s}$.
The desired velocity $\tilde{\vv}^{dyn}$ of the dynamic obstacle is computed as $\tilde{\vv}^{dyn} = \mM_r(\vp^{dyn} | \hat{\vc}, \hat{s}) = \frac{\vr}{\normtwo{\vr}}\hat{s}$ where $\vr\perp(\vp^{dyn}-\hat{\vc})$\footnote{During prediction, we assume that we have access to the algorithm computing $\vr$ from $\vp^{dyn}$ and $\hat{\vc}$ as there are infinitely many vectors perpendicular to $\vp^{dyn}-\vc$ when $d \geq 3$.}.
\end{itemize}
Let $\vp^{robot}$ be the current position and $\vv^{robot}$ be the current velocity of a robot.
We define one interaction model:
\begin{itemize}
\item \textbf{Repulsive interaction model $\boldsymbol{\mI_r(\vp^{dyn}, \tilde{\vv}^{dyn}, \vp^{robot}, \vv^{robot} | \hat{f})}$ (Fig.~\ref{Figure:Repulsive}):} Causes dynamic obstacle to be repulsed from the robot with repulsion strength $\hat{f}$.
The velocity of the dynamic obstacle is computed as $\vv^{dyn} = \mI_r(\vp^{dyn}, \tilde{\vv}^{dyn}, \vp^{robot}, \vv^{robot} | \hat{f}) = \tilde{\vv}^{dyn} + \frac{\left(\vp^{dyn} - \vp^{robot}\right)\hat{f}}{\normtwo{\vp^{dyn} - \vp^{robot}}^3}$.
The dynamic obstacle gets repulsed away from the robot linearly proportional to repulsion strength $\hat{f}$, and quadratically inversely proportional to the distance from the robot\footnote{Note that the interaction model we use does not utilize the velocity $\vv^{robot}$ of the robot, while our planner allows it. We choose to use this interaction model for easier online prediction of model parameters as our paper is not focused on prediction algorithms.}.
\end{itemize}

We implement three online prediction methods to predict the behavior models of dynamic obstacles from the sensed position and velocity histories of dynamic obstacles and the robot, one for each combination of movement and interaction models.
Here, we explain only one of the predictors for brevity since prediction is not the focus of our paper.
Each robot runs the prediction algorithms for each dynamic obstacle.
Let $\vp^{robot}_{hist}$ be the position and $\vv_{hist}^{robot}$ be the velocity history of the robot collected synchronously with $\vp^{dyn}_{hist}$ and $\vv^{dyn}_{hist}$ for the dynamic obstacle.\looseness=-1

\paragraph{Goal attractive repulsive predictor}
Assuming the dynamic obstacle moves according to goal attractive movement model $\mM_g(\vp^{dyn} | \hat{\vg}, \hat{s})$ and repulsive interaction model $\mI_r(\vp^{dyn}, \tilde{\vv}^{dyn}, \vp^{robot}, \vv^{robot} |\hat{f})$, we estimate parameters $\hat{\vg}$, $\hat{s}$ and $\hat{f}$.
We solve two consecutive quadratic programs (QP): i) one for goal $\hat{\vg}$ estimation, ii) one for desired speed $\hat{s}$ and repulsion strength $\hat{f}$ estimation\footnote{While joint estimation of $\hat{\vg}$, $\hat{s}$ and $\hat{f}$ would be more accurate, we choose to estimate the parameters in two steps so that the individual problems are QPs, and can be solved quickly.\looseness=-1}.\looseness=-1

\textbf{Goal estimation.} 
Let $\vp^{dyn}_{hist,k}$ and $\vv^{dyn}_{hist,k}$ be the $k$th elements of $\vp^{dyn}_{hist}$ and $\vv^{dyn}_{hist}$ respectively.
$\vp^{dyn}_{hist,k} + t_k\vv^{dyn}_{hist,k}, t_k \geq 0$ is the ray the dynamic obstacle would have followed if it did not change its velocity after the $k$th sample.
We estimate the goal position $\hat{\vg}$ of the dynamic obstacle by computing the point whose average squared distance to these rays 
is minimal:
\begin{align*}
    \min_{\hat{\vg}, t_1, \ldots, t_K} &\frac{1}{K} \sum_{k=1}^{K} \normtwo{\vp^{dyn}_{hist,k} + t_k\vv^{dyn}_{hist,k} - \hat{\vg}}^2, \text{s.t.}\\
    &t_k \geq 0\ \forall k\in \{1, \ldots, K\}
\end{align*}
where $K$ is the number of recorded position/velocity pairs.\looseness=-1

\textbf{Desired speed and repulsion strength estimation.} Assuming the dynamic obstacle moves according to the goal attractive repulsive behavior model, its estimated velocity at step $k$ is:

\begin{align*}
    \hat{\vv}^{dyn}_k &= \mI_r(\vp^{dyn}_{hist,k}, \mM_g(\vp^{dyn}_{hist,k} | \hat{\vg}, \hat{s}), \vp^{robot}_{hist,k}, \vv^{robot}_{hist,k} | \hat{f})\\
    &= \frac{\hat{\vg} - \vp^{dyn}_{hist,k}}{\normtwo{\hat{\vg}-\vp^{dyn}_{hist,k}}}\hat{s} + \frac{\left(\vp^{dyn}_{hist,k} - \vp^{robot}_{hist,k}\right)\hat{f}}{\normtwo{\vp^{dyn}_{hist,k} - \vp^{robot}_{hist,k}}^3}
\end{align*}
We minimize the average squared distance between estimated and sensed dynamic obstacle velocities to estimate $\hat{s}$ and $\hat{f}$:
\begin{align*}
    \min_{\hat{s},\hat{f}} \frac{1}{K}\sum_{k=1}^K \normtwo{\hat{\vv}^{dyn}_k - \vv^{dyn}_{hist,k}}^2.
\end{align*}\looseness=-1

\paragraph{Assigning probabilities to behavior models}
Each robot runs the three predictors for each dynamic obstacle.
For each predicted behavior model $(\mM_j, \mI_j)$, $j \in \{1,2,3\}$, we compute the average estimation error $E_j$ as the average $L^2$ norm between the predicted and the actual velocities:
\begin{align*}
    \frac{1}{K}\sum_{k=1}^K \lVert &\vv^{dyn}_{hist,k}  -\mI_j(\vp^{dyn}_{hist,k}, \mM_j(\vp^{dyn}_{hist,k}), \vp^{robot}_{hist,k}, \vv^{robot}_{hist,k})\rVert_2
\end{align*}
We compute the softmax of errors $E_j$ with base $b$ where $0 < b < 1$, and use them as probabilities, i.e., $p^{dyn}_{j} = \frac{b^{E_j}}{\sum_{k=1}^3 b^{E_k}}$.\looseness=-1

\subsubsection{Metrics}
We run each single robot simulation experiment $1000$ times and each multi-robot simulation experiment $100$ times in randomized environments in performance evaluation of our algorithm (Sec.~\ref{Section:PerfUnderDiffConfandEnv}).
In baseline comparisons, we run each single-robot simulation experiment $250$ times and each multi-robot simulation experiment $100$ times (Sec.~\ref{Section:BaselineComparison}).
In each experiment, the robots are tasked with navigating from their start to goal positions through an environment with static and dynamic obstacles.
There are nine metrics that we report for each experiment, averaged over all robots in all runs.\looseness=-1
\begin{itemize}
    \item \textbf{Success rate:} Ratio of robots that navigate from their start positions to their goal positions successfully without any collisions.
    \item \textbf{Collision rate:} Ratio of robots that collide with a static or a dynamic obstacle or a teammate at least once.
    \item \textbf{Deadlock rate:} Ratio of robots that deadlock, i.e., do not reach its goal position.\footnote{Note that, under this definition, livelocks, i.e., moving but not reaching to goal, and deadlocks in the classical sense, i.e., not moving at all, are both defined as deadlocks. We extend the definition of the deadlocks in order to decrease the number of terms we use during our discussion.}
    \item \textbf{Static obstacle collision rate:} Ratio of robots that collide with a static obstacle at least once.
    \item \textbf{Dynamic obstacle collision rate:} Ratio of robots that collide with a dynamic obstacle at least once.
    \item \textbf{Teammate collision rate:} Ratio of robots that collide with a teammate at least once.
    \item \textbf{Average navigation duration:} Average time it takes for a robot to navigate from its start position to its goal position across no-deadlock no-collision robots.
    \item \textbf{Planning fail rate:} Ratio of failing planning iterations over all planning iterations of all robots in all runs.
    \item \textbf{Average planning duration:} Average planning duration over all planning iterations of all robots in all runs.
\end{itemize}\looseness=-1

\subsubsection{Fixed Parameters and Run Randomization}

Here, we describe fixed parameters across all experiments and the parameters that are randomized in all experiments.
Fixed parameters are shared by each robot $i$, and randomized parameters are randomized the same way for all robots $i$.\looseness=-1

\textbf{Fixed parameters.} We set $p_i^{min}=0.1$, $\tau_i=\SI{2.5}{s}$, $\tilde{\gamma}_i^1=\SI{5.0}{\frac{m}{s}}$, $\tilde{\tau}_i=\SI{2.0}{s}$, $\alpha_i=1.5$,  $T^{search}_i=\SI{75}{ms}$, $h_{i,l}=13$ for all $l$,  $\theta_{i,l}$ and $\beta_{i,l}$ $10, 20, 30$ for the first three pieces, and $40$ for the remaining pieces, $\lambda_{i,1} = 2.8$, $\lambda_{i,2} = 4.2$, $\lambda_{i,4} = 0.2$, and $\lambda_{i,l} = 0$ for all other degrees, $c_i=2$,  maximum velocity $\gamma_i^1=\SI{10}{\frac{m}{s}}$ and maximum acceleration $\gamma_i^2 = \SI{15}{\frac{m}{s^2}}$ for all robots $i$.
The FORWARD actions of search are ($\SI{2.0}{\frac{m}{s}}$, $\SI{0.5}{s}$), ($\SI{3.5}{\frac{m}{s}}$, $\SI{0.5}{s}$), and ($\SI{4.5}{\frac{m}{s}}$, $\SI{0.5}{s}$).\looseness=-1

The desired trajectory of each robot is set to the shortest path connecting its start to its goal position, avoiding only the static obstacles.
The duration of the desired trajectory assumes the robot follows it at $\frac{1}{3}$ of its maximum speed $\tilde{\gamma}_i^1$ for search.\looseness=-1

In all runs of all experiments, robots navigate in random forest environments, i.e., static obstacles are tree-like objects.
The forest has $\SI{15}{m}$ radius, and trees are $\SI{6}{m}$ high and have a radius of $\SI{0.5}{m}$.
The forest is centered around the origin.
The octree structure has a resolution of $\SI{0.5}{m}$.
The density $\rho$ of the forest, i.e., the ratio of occupied cells in the octree within the forest, is set differently in each experiment.\looseness=-1

\textbf{Run randomization.}
We randomize the following parameters in all runs of each experiment in the same way.\looseness=-1



\textit{Dynamic obstacle randomization.} 
We randomize the axis-aligned box collision shape of each dynamic obstacle by randomizing its size in each dimension uniformly in $[\SI{1}{m}, \SI{4}{m}]$. 
The dynamic obstacle's initial position is uniformly sampled in the box $A$ with minimum corner $\begin{bmatrix}\SI{-12}{m} & \SI{-12}{m} & \SI{-2}{m}\end{bmatrix}^\top$ and maximum corner $\begin{bmatrix}\SI{12}{m} & \SI{12}{m} & \SI{6}{m}\end{bmatrix}^\top$.
We sample the movement model of the obstacle among goal attractive, constant velocity, and rotating models.
If the goal attractive movement model is sampled, we sample its goal position $\hat{\vg}$ uniformly in the box $A$.
If rotating model is sampled, we sample the rotation center $\hat{\vc}$ in the box with minimum corner $\begin{bmatrix}\SI{-0.5}{m} & \SI{-0.5}{m} & \SI{0.0}{m}\end{bmatrix}^\top$ and maximum corner $\begin{bmatrix}\SI{0.5}{m} & \SI{0.5}{m} & \SI{6.0}{m}\end{bmatrix}^\top$.
The desired speed $\hat{s}$ of the obstacles is sampled uniformly in $[\SI{0.5}{\frac{m}{s}}, \SI{1.0}{\frac{m}{s}}]$.
If the constant velocity model is sampled, the velocity $\hat{\vv}$ is set by uniformly sampling a unit direction vector and multiplying it with $\hat{s}$.
The interaction model is always the repulsive model. 
The repulsion strength $\hat{f}$ is set/randomized differently in each experiment.\footnote{Note that, in reality, there is no necessity that the dynamic obstacles move according to the behavior models that we define.
The reason we choose to move dynamic obstacles according to these behavior models is so that at least one of our predictors assumes the correct model for the obstacle.
The prediction is still probabilistic in the sense that we generate three hypotheses for each dynamic obstacle and assign probabilities.\looseness=-1
}
Each dynamic obstacle changes its velocity every decision-making period, which is sampled uniformly from $[\SI{0.1}{s}, \SI{0.5}{s}]$.
Each dynamic obstacle runs its movement model to get its desired velocity and runs its interaction model for each robot, and executes the average velocity, i.e., dynamic obstacles interact with all robots, while an individual robot models only the interactions with itself.
The number of dynamic obstacles $\#^D$ is set differently in each experiment.\looseness=-1


\textit{Robot randomization.}
We randomize the axis-aligned box collision shape of each robot by randomizing its size in each dimension uniformly in $[\SI{0.2}{m}, \SI{0.3}{m}]$.
We sample the replanning period of each robot uniformly in $[\SI{0.2}{s}, \SI{0.4}{s}]$.
Robot start positions are selected randomly around the forest on a circle with radius $\SI{21.5}{m}$ at height $\SI{2.5}{m}$.
They are placed on the circle with equal arc length distance between them.
The goal positions are set to the antipodal points of the start positions on the circle.
The number of robots $\#^R$ is set/randomized differently in each experiment.\looseness=-1


Sample environments with varying static obstacle densities and the number of dynamic obstacles are shown in Fig.~\ref{Figure:SampleEnvironments}.

\begin{figure}
    \centering
    \subfloat[$\rho = 0.1, \#^D = 0$]{\includegraphics[width=0.45\linewidth, height=3cm]{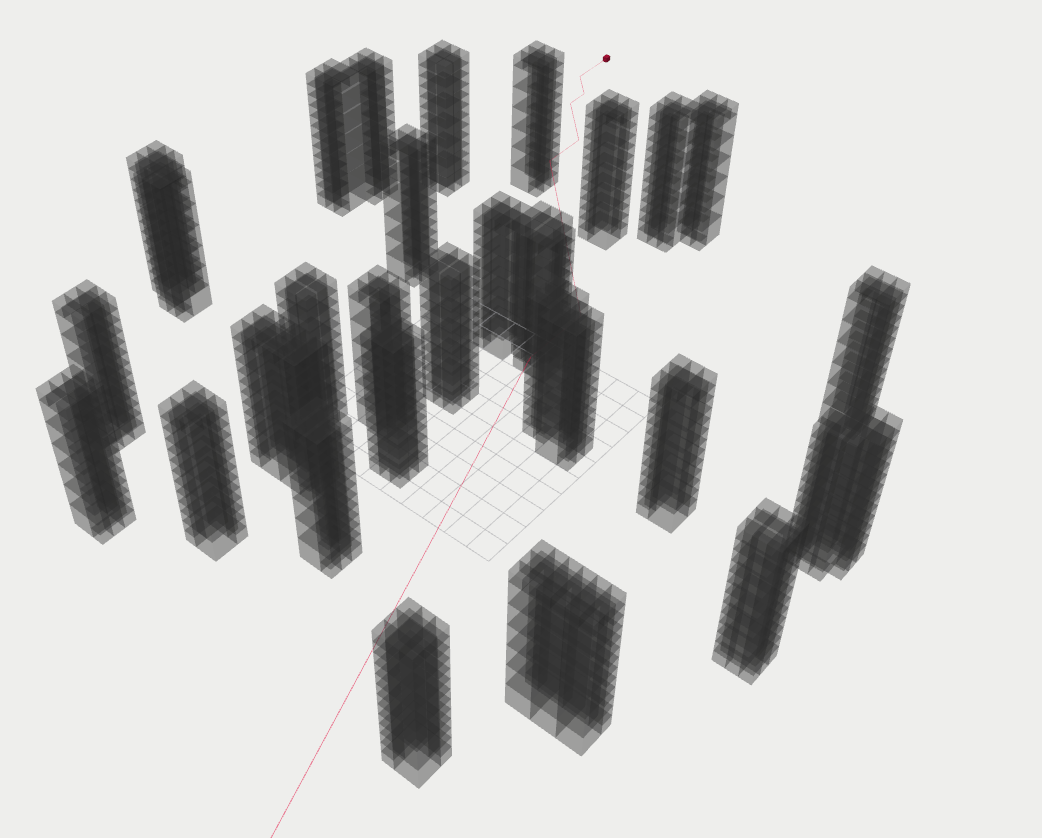}}
    \hfill
    \subfloat[$\rho = 0.2, \#^D = 100$]{\includegraphics[width=0.45\linewidth, height=3cm]{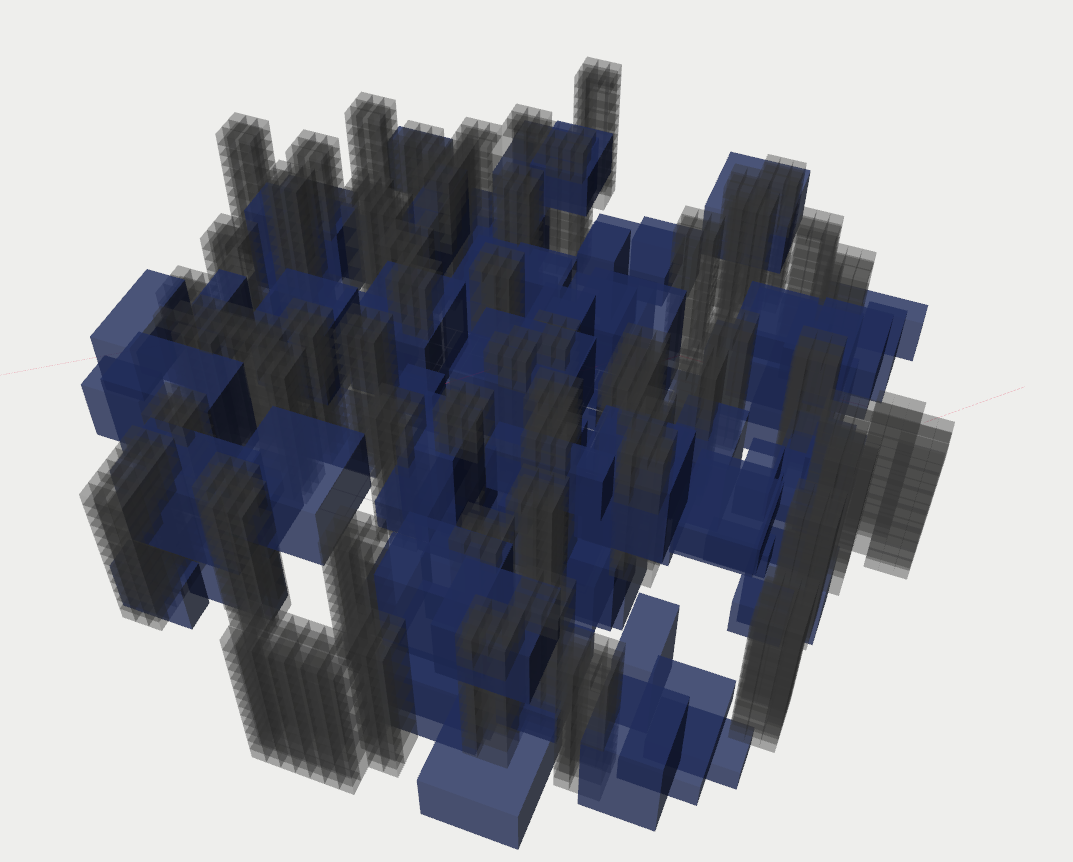}}
    \hfill
    \caption{Sample environments generated during run randomization with varying static obstacle density $\rho$ and number of dynamic obstacles $\#^D$. {\color{darkgray}Gray} objects are static obstacles representing the forest. {\color{blue}Blue} objects are dynamic obstacles.}
    \vspace{-0.1in}
    \label{Figure:SampleEnvironments}
\end{figure}




\subsubsection{Simulating Communication Imperfections}\label{Section:SimulatingCommunicationImperfections}

Robots communicate with each other in order to update tail time points, whenever they successfully plan.
We simulate imperfections in the communication medium.
We model message delays as an exponential distribution with mean $\delta$ seconds and message drops with a Bernoulli distribution with drop probability $\kappa$ similar to~\cite{senbaslar2022async}.
Each message a robot broadcasts is dropped with probability $\kappa$, and applied a delay sampled from the delay distribution if it is not dropped.
Message re-orderings are naturally introduced by random delays.
\looseness=-1

\subsection{Performance under Different Configurations and Environments in Simulations}\label{Section:PerfUnderDiffConfandEnv}

We evaluate the performance of DREAM when it is used in different environments and configurations.\looseness=-1

\subsubsection{Repulsive Dynamic Obstacle Interactivity}

\begin{table}[t]
    \centering
    \caption{Effects of repulsion strength}
    \vspace{-0.05in}
    \label{Table:RepulsionStrength}
    \resizebox{\columnwidth}{!}{\begin{tabular}{|c|c|c|c|c|c|c|}
         \hline $\hat{f}$ & {\color{cyan}$-0.5$} & $0$ & $0.5$ & $1.5$ & $3.0$ & $6.0$  \\
         \hline 
         \textbf{succ. rate} & 0.818 & 0.897 & 0.915 & 0.936 & 0.970 & 0.991\\ 
         \textbf{coll. rate} & 0.182 & 0.102  & 0.084 & 0.064 & 0.030  & 0.009\\
         \textbf{deadl. rate} & 0.001 & 0.004 & 0.002 & 0.000  &  0.000  & 0.000 \\
         {\color{lightgray}\textbf{s. coll. rate}} & {\color{lightgray}0.000} & {\color{lightgray}0.000} & {\color{lightgray}0.000} & {\color{lightgray}0.000} &   {\color{lightgray}0.000} &{\color{lightgray}0.000} \\
         \textbf{d. coll. rate}  & 0.182 & 0.102  & 0.084  & 0.064 &   0.030 & 0.009\\
          {\color{lightgray}\textbf{t. coll. rate}} &  {\color{lightgray}0.000} &  {\color{lightgray}0.000} &  {\color{lightgray}0.000} &  {\color{lightgray}0.000} &  {\color{lightgray}0.000} &  {\color{lightgray}0.000} \\
         \textbf{avg. nav. dur. [s]}   & 26.71 & 26.72 & 26.71 & 26.71 & 26.69  & 26.68\\
         \textbf{pl. fail rate} & 0.029 & 0.030 & 0.026 & 0.022 & 0.019  & 0.014\\
         \textbf{avg. pl. dur. [ms]}  & 49.72 & 50.63 & 50.55  & 46.98  &  43.17 & 37.76\\
         \hline
    \end{tabular}}
    \vspace{-0.25in}
\end{table}

We experiment with different levels of repulsive interactivity of dynamic obstacles and we compare navigation metrics when dynamic obstacles use different repulsion strengths $\hat{f}$ in single-robot experiments.
During these experiments, we set $\rho_i = 0$, and $\#^D = 50$.
The results are summarized in Table~\ref{Table:RepulsionStrength}.\looseness=-1

In general, as the repulsive interactivity increases, the collision and deadlock rates decrease, and the success rate increases.
In addition, the average planning duration decreases as the repulsive interactivity increases, because the problem gets easier for the robot if dynamic obstacles take some responsibility for collision avoidance even with a simple repulsion rule.
In the first experiment, repulsion strength is set to $-0.5$ ({\color{cyan}cyan} value), causing dynamic obstacles to get attracted to the robot, i.e., they move towards the robot.
Even in that case, the robot can achieve a high success rate, i.e., 0.818, as it models dynamic obstacle interactivity.\looseness=-1

In the remaining experiments, we sample $\hat{f}$ in $[0.2, 0.5]$, unless explicitly stated otherwise.

\subsubsection{Dynamic Obstacle Sensing Uncertainty}\label{Section:DynamicObstacleSensingUncertainty}

\begin{table}[t]
    \centering
    \caption{Effects of dynamic obstacle sensing uncertainty}
    \vspace{-0.05in}
    \label{Table:DynamicObstacleSensingUncertainty}
    \resizebox{\columnwidth}{!}{\begin{tabular}{|c|c|c|c|c|c|c|}
         \hline $\Sigma_i [\times I_{2d}]$ & $0.0$ & $0.1$ & $0.1$ & $0.2$ & $0.2$ & $0.5$  \\
         $\sigma$ & $0.0$ & $0.0$ & $0.1$ & $0.1$ & $0.2$ & $0.5$  \\
         \hline 
         \textbf{succ. rate} & 0.905 & 0.879 & 0.832 & 0.797 & {\color{cyan}0.661} & 0.410\\ 
         \textbf{coll. rate} &  0.094 & 0.119 &  0.168 & 0.203 & 0.339 &  0.590\\
         \textbf{deadl. rate} &  0.001& 0.003 & 0.002 & 0.001&0.004  & 0.009 \\
         {\color{lightgray}\textbf{s. coll. rate}} & {\color{lightgray}0.000} & {\color{lightgray}0.000} & {\color{lightgray}0.000} & {\color{lightgray}0.000} & {\color{lightgray}0.000}  &  {\color{lightgray}0.000}\\
         \textbf{d. coll. rate}& 0.094 & 0.119 & 0.168 & 0.203& 0.339 &  0.590\\
          {\color{lightgray}\textbf{t. coll. rate}} &  {\color{lightgray}0.000} &  {\color{lightgray}0.000} &  {\color{lightgray}0.000} &  {\color{lightgray}0.000} &  {\color{lightgray}0.000} &  {\color{lightgray}0.000} \\
         \textbf{avg. nav. dur. [s]} & 26.71 & 26.70 & 26.70 &26.71 &26.69  &  26.72\\
         \textbf{pl. fail rate}& 0.026 & 0.030 & 0.028 & 0.029& 0.031 & 0.037 \\
         \textbf{avg. pl. dur. [ms]}&  59.76& 55.53 & 54.54 & 54.83& 54.46 & 56.15 \\
         \hline
    \end{tabular}}
    \vspace{-0.2in}
\end{table}

The dynamic obstacle sensing uncertainty is modeled by i) applying a zero mean Gaussian with covariance $\Sigma_i$ to the sensed positions and velocities of dynamic obstacles and ii) randomly inflating dynamic obstacle shapes by a zero mean Gaussian with standard deviation $\sigma_i$.
These create three inaccuracies reflected to our planner: i) dynamic obstacle shapes provided to our planner become wrong, ii) prediction inaccuracy increases, and iii) current positions $\vp^{dyn}_{i,j}$ of dynamic obstacles provided to the planner become wrong.
Our planner models uncertainty across behavior models by using realization probabilities but does not explicitly account for the current position or shape uncertainty.\looseness=-1

To evaluate the performance of DREAM under different levels of dynamic obstacle sensing uncertainty, we control $\Sigma_i$ and $\sigma_i$ and report the metrics in single-robot experiments.
During these experiments, we set $\rho = 0$ to evaluate the effects of dynamic obstacles only. 
We set $\#^D = 50$.\looseness=-1

The results are given in Table~\ref{Table:DynamicObstacleSensingUncertainty}.
$\Sigma_i$ is set to a constant multiple of identity matrix $I_{2d}$ of size $2d \times 2d$ in each experiments.
Expectedly, as the uncertainty increases, the success rate decreases.
An increase in collision and deadlock rates is also seen.
Similarly, the planning failure rate tends to increase as well.\looseness=-1

\begin{table}[t]
    \centering
    \caption{Effects of dynamic obstacle inflation under dynamic obstacle sensing uncertainty of $\Sigma_i = 0.2I_{2d}, \sigma_i = 0.2$.}
    \vspace{-0.05in}
    \label{Table:InflationUnderDynamicObstacleSensingUncertainty}
    \resizebox{\columnwidth}{!}{\begin{tabular}{|c|c|c|c|c|c|c|}
         \hline inflation $[m]$ & $0.2$ & $0.5$ & $1.0$ & $1.5$ & $2.0$ & $4.0$  \\
         \hline 
         \textbf{succ. rate} & {\color{cyan}0.779} & {\color{cyan}0.818} & 0.747 & 0.617 & 0.417 & {\color{red}0.057}\\ 
         \textbf{coll. rate} & 0.220 & 0.182 & 0.250 & 0.382 & 0.581 &  0.943\\
         \textbf{deadl. rate}& 0.003 & 0.001 &  0.005& 0.007 & 0.003 &  0.000\\
          {\color{lightgray}\textbf{s. coll. rate}} &  {\color{lightgray}0.000} &  {\color{lightgray}0.000} &  {\color{lightgray}0.000} &  {\color{lightgray}0.000}  &  {\color{lightgray}0.000} &   {\color{lightgray}0.000}\\
         \textbf{d. coll. rate} & 0.220 & 0.182 & 0.250 & 0.382 & 0.581 &  0.943\\
          {\color{lightgray}\textbf{t. coll. rate}} &  {\color{lightgray}0.000} &  {\color{lightgray}0.000} &  {\color{lightgray}0.000} &  {\color{lightgray}0.000} &  {\color{lightgray}0.000} &  {\color{lightgray}0.000} \\
         \textbf{avg. nav. dur. [s]} & 26.70 & 26.72 & 26.93 & 27.16 & 27.79 & 26.81 \\
         \textbf{pl. fail rate} & 0.032 & 0.035 & 0.047 & 0.058 &0.067 &  0.021\\
         \textbf{avg. pl. dur. [ms]} & 57.80 &  63.014& 70.92 & 74.63 & 72.66 & 23.13 \\
         \hline
    \end{tabular}}
    \vspace{-0.25in}
\end{table}

One common approach to tackling unmodeled uncertainty for obstacle avoidance is artificially inflating the shapes of obstacles.
To show the effectiveness of this approach, we set $\Sigma_i = 0.2I_{2d}$ and $\sigma_i = 0.2$, and inflate the shapes of obstacles with different amounts during planning.
The results of these experiments are given in Table~\ref{Table:InflationUnderDynamicObstacleSensingUncertainty}.
Inflation clearly helps when done in reasonable amounts.
When inflation is set to $\SI{0.2}{m}$, success rate increases from $0.661$ as reported in Table~\ref{Table:DynamicObstacleSensingUncertainty} ({\color{cyan}cyan} value) to $0.779$.
It further increases to $0.818$ when the inflation amount is set to $\SI{0.5}{m}$, which can be seen in {\color{cyan}cyan} values in Table~\ref{Table:InflationUnderDynamicObstacleSensingUncertainty}.
However, as the inflation amount increases, metrics start to degrade as the planner becomes overly conservative.
Success rate decreases down to $0.057$ when the inflation amount is set to $\SI{4.0}{m}$ as it can be seen in the {\color{red}red} value in Table~\ref{Table:InflationUnderDynamicObstacleSensingUncertainty}.\looseness=-1
 

\subsubsection{Static Obstacle Sensing Uncertainty}\label{Section:StatiObstacleSensingUncertainty}

\begin{table}
    \centering
    \caption{Effects of static obstacle sensing uncertainty}
    \vspace{-0.05in}
    \label{Table:StaticObstacleSensingUncertainty}
    
    \resizebox{\linewidth}{!}{%
    \begin{tabular}{|c|c|c|c|c|c|c|}
         \hline imperfections & $L(0.2)$ & $L(0.2)^2$ & $L(0.3)^2I$ &$L(0.3)^2ID$& $L(0.3)^2IDL(0.2)$ & $L(0.3)^2IDL(0.5)$   \\
         \hline 
         \textbf{succ. rate} & 0.994 & 0.992 & 0.986 & 0.956 & 0.971 & 0.949 \\
         \textbf{coll. rate} & {\color{cyan}0.001}  & {\color{cyan}0.003}&{\color{cyan}0.005} &  {\color{red}0.037}&{\color{orange}0.020} & {\color{magenta}0.039}\\
         \textbf{deadl. rate}&{\color{cyan}0.005} &{\color{cyan}0.005} &{\color{cyan}0.009} & 0.007& 0.009 & 0.015\\
         \textbf{s. coll. rate}&0.001 &0.003 & 0.005& 0.037&0.020 &0.039 \\
         {\color{lightgray}\textbf{d. coll. rate}}&{\color{lightgray}0.000} & {\color{lightgray}0.000}& {\color{lightgray}0.000}&{\color{lightgray}0.000} & {\color{lightgray}0.000}&{\color{lightgray}0.000} \\
          {\color{lightgray}\textbf{t. coll. rate}} &  {\color{lightgray}0.000} &  {\color{lightgray}0.000} &  {\color{lightgray}0.000} &  {\color{lightgray}0.000} &  {\color{lightgray}0.000} &  {\color{lightgray}0.000} \\
         \textbf{avg. nav. dur. [s]}&27.49 & 27.53& 27.73&27.49 & 27.55& 28.18\\
         \textbf{pl. fail rate}&0.029 &0.041 & 0.053& 0.038 & 0.047 & 0.063\\
         \textbf{avg. pl. dur. [ms]}& 36.16& 41.76&53.39 &40.58 & 46.41 &67.65 \\
         \hline
    \end{tabular}}
    \vspace{-0.15in}
\end{table}

As we describe in Sec.~\ref{Section:Mocking}, we use three operations to model sensing imperfections of static obstacles: i) increaseUncertainty, ii) leakObstacles($p_{leak}$), and iii) deleteObstacles.
We evaluate the effects of static obstacle sensing uncertainty by applying a sequence of these operations to the octree representation of static obstacles, and provide the resulting octree to our planner in single-robot experiments.
Application of leakObstacles increases the density of the map, but the resulting map contains the original obstacles. 
increaseUncertainty does not change the density of the map but increases the uncertainty associated with the obstacles.
deleteObstacles decreases the density of the map, but the resulting map may not contain the original obstacles, leading to unsafe behavior. In these experiments, we set $\rho = 0.1$, and $\#^D = 0$.\looseness=-1

The results are given in Table~\ref{Table:StaticObstacleSensingUncertainty}.
In the imperfections row of the table, we use $L(p_{leak})$ as an abbreviation for leakObstacles($p_{leak}$), and $L(p_{leak})^n$ as an abbreviation for repeated application of leakObstacles to the octree.
We use $I$ for increaseIncertainty, and $D$ for deleteObstacles.
Leaking obstacles or increasing the uncertainty associated with them does not increase the collision and deadlock rates significantly as seen in the {\color{cyan}cyan} values.
The planning duration and failure rates increase as the number of obstacles increases.
Deleting obstacles causes a sudden jump of the collision rate as it can be seen between {\color{cyan}cyan} and {\color{red}red} values because the planner does not know about existing obstacles.
Leaking obstacles back with $p_{leak} = 0.2$ after deleting them decreases the collision rate back as it can be seen between {\color{red}red} and {\color{orange}orange} values.
However, leaking obstacles with high probability increases the collision rate back as it can be seen between {\color{orange}orange} and {\color{magenta}magenta} values.
This happens because the environments get significantly more complicated because the number of obstacles increases.
Environment complexity is also reflected in the increased deadlock rate.\looseness=-1

\subsubsection{Imperfect Communication Medium}

\begin{table}
    \centering
    \caption{Effects of imperfect communication medium}
    \vspace{-0.05in}
    \label{Table:ImperfectCommunicationMedium}
    
    \resizebox{\linewidth}{!}{%
    \begin{tabular}{|c|c|c|c|c|c|c|c|c|}
         \hline $\delta[s]$ & $0.0$ & $0.0$ & $0.25$ &$0.25$& $1.0$ & $1.0$ & $5.0$ & $5.0$   \\
         $\kappa$ & $0.0$ & $0.0$ & $0.1$ &$0.1$& $0.25$ & $0.25$ & $0.75$ & $0.75$   \\
         $\#^D$ & $0$ & $50$ & $0$ &$50$& $0$ & $50$ & $0$ & $50$   \\
         \hline 
         \textbf{succ. rate} & 1.000 & 0.952 & 1.000 & 0.956 & 0.995 & 0.923 & 0.490 & 0.702\\
         \textbf{coll. rate}& 0.000 & 0.048 & 0.000 & 0.044 & 0.000 & 0.076 & 0.000 & 0.289\\
         \textbf{deadl. rate} & 0.000 & 0.000 & 0.000 & 0.000 &0.005  &0.001 &{\color{magenta}0.510} & 0.013\\
         {\color{lightgray}\textbf{s. coll. rate}}& {\color{lightgray}0.000} & {\color{lightgray}0.000} & {\color{lightgray}0.000} & {\color{lightgray}0.000} & {\color{lightgray}0.000} & {\color{lightgray}0.000} & {\color{lightgray}0.000} & {\color{lightgray}0.000} \\
         \textbf{d. coll. rate}& 0.000 & 0.045 & 0.000 & 0.043 & 0.000 & 0.072& 0.000 & 0.250\\
          \textbf{t. coll. rate}& {\color{teal}0.000} & {\color{cyan}0.004} & {\color{teal}0.000} & {\color{cyan}0.003} & {\color{teal}0.000} & {\color{cyan}0.006} & {\color{teal}0.000} & 0.067\\
         \textbf{avg. nav. dur. [s]}& 60.45 & 45.35 & 65.97 & 50.18 & 82.49 & 59.21 & {\color{magenta}128.97} & 79.06\\
         \textbf{pl. fail rate}& $<$0.001 & 0.005 & $<$0.001 & 0.006 & $<$0.001 & 0.008 & $<$0.001 & 0.037\\
         \textbf{avg. pl. dur. [ms]}& 168.90 & 164.06 & 249.36 & {\color{red}208.89} & {\color{orange}382.10} & 297.25 & {\color{orange}1042.66} & {\color{orange}768.56}\\
         \hline
    \end{tabular}}
    \vspace{-0.2in}
\end{table}

We evaluate the performance of DREAM in multi-robot experiments with or without dynamic obstacles under different levels of communication imperfections. 
In these experiments, we set static object density $\rho = 0$, and teammate safety enforcement duration $T^{team}_i = \infty$, i.e., we enforce DSHTs for the full plans.
The obstacles are not interactive, i.e., $\hat{f} = 0$.
The number of robots is $\#^R = 16$.
We control the average delay $\delta$, message drop probability $\kappa$, and number of dynamic obstacles $\#^D$.
The results are given in Table~\ref{Table:ImperfectCommunicationMedium}.\looseness=-1

DREAM results in no collisions between teammates when there are no obstacles regardless of the imperfection amount of the communication medium as seen in {{\color{teal}teal} values.
However, when dynamic obstacles are present, teammates may collide with each other as it can be seen in {\color{cyan}cyan} values, since we prioritize dynamic obstacle avoidance to teammate avoidance.
As communication imperfections increase, DREAM becomes more and more conservative for teammate safety.
When there are no dynamic obstacles, this causes conservative behavior, increasing average navigation duration as well as the deadlock rate as it can be seen in {\color{magenta}magenta} values.
As cardinalities of $\tilde{\mH}^{active}_i$ increase in these scenarios, the number of constraints increase substantially in trajectory optimization, slowing down the algorithm such that it is not real-time anymore as it can be seen in {\color{orange}orange} values.
When obstacles are present, deadlocks occur less frequently as the primary objective becomes obstacle avoidance, causing robots to disperse in the environment, and hence, avoid teammates easily.\looseness=-1



\subsubsection{Teammate Safety Enforcement Duration}\label{Section:TeammateSafetyEnforcementDurationEvaluation}

\begin{table}[t]
    \centering
    \caption{Effects of teammate safety duration}
    \vspace{-0.05in}
    \label{Table:SafetyEnforcementDuration}
    \resizebox{\columnwidth}{!}{\begin{tabular}{|c|c|c|c|c|c|c|}
         \hline $\boldsymbol{T_{team}}$ & $0.0$ & $0.5$ & $1.0$ & $1.5$ & $2.0$ & $2.5$\\
        \hline 
         \textbf{succ. rate} &0.300&0.748& {\color{red}0.964} &0.963&0.961&0.933\\ 
         \textbf{coll. rate} &0.700&0.252& 0.036&0.038&0.039&0.068\\ 
         \textbf{deadl. rate} &0.000&0.001& 0.000&0.000&0.000&0.000\\ 
         \textbf{s. coll. rate} &{\color{lightgray}0.000}&{\color{lightgray}0.000}&{\color{lightgray}0.000}&{\color{lightgray}0.000}&{\color{lightgray}0.000}&{\color{lightgray}0.000}\\ 
         \textbf{d. coll. rate} &0.068&0.098&0.036&0.038&0.036&0.058\\ 
         \textbf{t. coll. rate} &0.677&0.183&0.001&0.000&0.005&0.010\\ 
         \textbf{avg. nav. dur. [s]} &{\color{orange}27.57}&{\color{orange}30.55}&{\color{orange}37.60}&{\color{orange}44.71}&{\color{orange}51.01}&{\color{orange}49.38}\\ 
         \textbf{pl. fail rate} &0.011&0.025&0.008&0.005&0.005&0.006\\ 
         \textbf{avg. pl. dur. [ms]} &46.58&77.87&{\color{magenta}115.30}&153.47&195.93&205.35\\ 
         \hline
    \end{tabular}}
    \vspace{-0.2in}
\end{table}

Enforcing teammate safety for the full trajectories causes planning to be not real-time and results in a high rate of deadlocks when communication imperfections are high. We investigate the effects of relaxing teammate constraints by controlling safety enforcement durations $T^{team}_i$.
In these experiments, we set $\rho = 0$, $\#^D=50$, $\hat{f}=0$, $\delta = \SI{0.25}{s}$, $\kappa = 0.1$, and $\#^R=16$.
Hence, these experiments can be compared with forth column in Table~\ref{Table:ImperfectCommunicationMedium}.
The results are given in Table~\ref{Table:SafetyEnforcementDuration}.\looseness=-1

The average navigation duration of the robots tends to increase as $T_i^{team}$ increases, because the planner becomes more and more conservative, as it can be seen in {\color{orange}orange} values.
Setting $T_{team}^i=1.0$ results in the best success rate, $0.964$, as seen in the {\color{red}red} value.
The average planning duration decreases by $\approx 50\%$ when $T_{team}^i$ is set to $1.0$ compared to setting it $\infty$, greatly increasing the cases our algorithm is real-time as it can be seen in the {\color{magenta}magenta} value compared to the {\color{red}red} value in Table~\ref{Table:ImperfectCommunicationMedium}.
The success rate of $0.964$ is higher than that of forth column in Table~\ref{Table:ImperfectCommunicationMedium}, which is $0.956$.
In addition, the average navigation duration decreases to $37.60$ compared to $50.18$ in Table~\ref{Table:ImperfectCommunicationMedium}.
Relaxing safety with respect to teammates not only decreases planning and navigation durations but improves the success rate of the algorithm.
In the remaining experiments, we set $T_i^{team} = 1.0$.\looseness=-1

\subsection{Baseline Comparisons in Simulations}\label{Section:BaselineComparison}

We summarize the collision avoidance and deadlock-free navigation comparisons between state-of-the-art decentralized navigation decision-making algorithms in Fig.~\ref{Figure:BaselineComparison}.
The comparison graph suggests that SBC~\cite{wang2017safety}, RLSS~\cite{senbaslar2023rlss}, RMADER~\cite{kondo2023rmader}, and TASC~\cite{toumieh2022tasc,toumieh2023tasc} are the best performing algorithms.
TASC and RMADER are shown to have a similar collision avoidance performance in~\cite{toumieh2023tasc}.
Therefore, we compare our algorithm DREAM to SBC, RLSS, and RMADER, and establish that it results in a better performance than them.
SBC is a short horizon algorithm: it computes the next safe acceleration to execute to drive the robot to its goal position each iteration.
RMADER, RLSS, and DREAM are medium horizon algorithms: they compute long trajectories, which they execute for a shorter duration in a receding horizon fashion.\looseness=-1

All of our baselines drive robots to their given goal positions.
We add support for desired trajectories by running our goal selection algorithm in every planning iteration and providing the selected goal position as intermediate goal positions.
We do not introduce static obstacle sensing uncertainty during baseline comparisons as none of the baselines explicitly account for it.
\looseness=-1

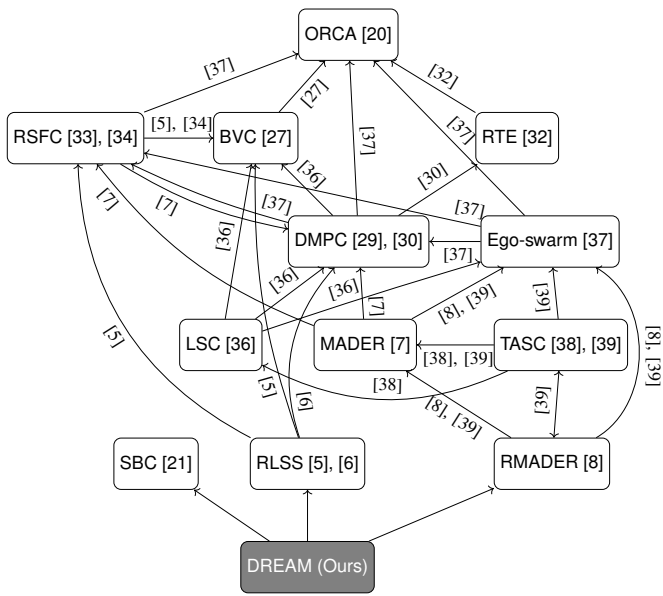
\begin{figure}[t]
\centering
\resizebox{\linewidth}{!}{%
\begin{tikzpicture}

    \node[algorithm] (ORCA) {\shortstack{ORCA~\cite{van2011reciprocal}}};
    \node[algorithm, below left=1.0cm and 3.0cm of ORCA] (RSFC) {\shortstack{RSFC~\cite{park2021rsfc,park2020rsfc}}};
    \node[algorithm, below left=1.0cm and 0.0cm of ORCA] (BVC) {\shortstack{BVC~\cite{zhou2017bvc}}};
    \node[algorithm, below right=1.0cm and 1.5cm of ORCA] (RTE) {\shortstack{RTE~\cite{senbaslar2018rte}}};
    \node[algorithm, below right=1.0cm and -0.2cm of BVC] (DMPC) {\shortstack{DMPC~\cite{luis2019dmpc,luis2020dmpc}}};
    \node[algorithm, right=1.0cm of DMPC] (Ego-swarm) {\shortstack{Ego-swarm~\cite{zhou2021ego}}};
    \node[algorithm, below left=1.0cm and 0.5cm of DMPC] (LSC) {\shortstack{LSC}~\cite{park2022lsc}};
    \node[algorithm, right=1.0cm of LSC] (MADER) {\shortstack{MADER~\cite{tordesillas2020mader}}};
    \node[algorithm, below left=1.3cm and -1.0cm of MADER] (RLSS) {\shortstack{RLSS~\cite{senbaslar2021rlss,senbaslar2023rlss}}};
    \node[algorithm, below right=1.3cm and 1.5cm of MADER] (RMADER) {\shortstack{RMADER~\cite{kondo2023rmader}}};
    \node[algorithm, left=1.0cm of RLSS] (SBC) {\shortstack{SBC~\cite{wang2017safety}}};
    \node[algorithm, right=1.5cm of MADER] (TASC) {\shortstack{TASC~\cite{toumieh2022tasc,toumieh2023tasc}}};
    \node[ours, below=1.0cm of RLSS] (DREAM) {\shortstack{DREAM (Ours)}};

    \path[->] (RTE) edge node[sloped,above] {\cite{senbaslar2018rte}} (ORCA);
    \path[->] (BVC) edge node[sloped,below] {\cite{zhou2017bvc}} (ORCA);
    \path[->] (RSFC) edge node[sloped,above] {\cite{zhou2021ego}} (ORCA);
    \path[->] (RMADER) edge node[sloped,below] {\cite{kondo2023rmader,toumieh2023tasc}} (MADER);
    \path[->, bend right=1500] (RMADER) edge node[sloped,above] {\cite{kondo2023rmader,toumieh2023tasc}} (Ego-swarm);
    \path[->] (MADER) edge node[sloped,below] {\cite{kondo2023rmader,toumieh2023tasc}} (Ego-swarm);
    \path[->] (RSFC) edge node[sloped,above] {\cite{park2021rsfc,senbaslar2023rlss}} (BVC);
    \path[->] (DMPC) edge node[sloped,above] {\cite{luis2020dmpc}} (RTE);
    \path[->, bend left=15] (MADER) edge node[sloped,anchor=north,pos=0.9] {\cite{tordesillas2020mader}} (RSFC);
    \path[->] (MADER) edge node[sloped,anchor=south,pos=0.3] {\cite{tordesillas2020mader}} (DMPC);
    \path[->,bend right=10] (RSFC) edge node[sloped,anchor=north,pos=0.3] {\cite{tordesillas2020mader}} (DMPC);
    \path[->,bend left=30] (RLSS) edge node[sloped,anchor=south,pos=0.2] {\cite{senbaslar2021rlss}} (DMPC);
    \path[->,bend left=30] (RLSS) edge node[sloped,below] {\cite{senbaslar2023rlss}} (RSFC);
    \path[->, bend left=10] (RLSS) edge node[sloped,anchor=north,pos=0.2] {\cite{senbaslar2023rlss}} (BVC);
    \path[->] (Ego-swarm) edge node[sloped,below] {\cite{zhou2021ego}} (DMPC);
    \path[->] (Ego-swarm) edge node[sloped,anchor=south,pos=0.05] {\cite{zhou2021ego}} (RSFC);
    \path[->] (Ego-swarm) edge node[sloped,above] {\cite{zhou2021ego}} (ORCA);
    \path[->,bend left=5] (DMPC) edge node[sloped,anchor=south,pos=0.1] {\cite{zhou2021ego}} (RSFC);
    \path[->] (DMPC) edge node[sloped,above] {\cite{zhou2021ego}} (ORCA);
    \path[->] (LSC) edge node[sloped,above] {\cite{park2022lsc}} (DMPC);
    \path[->] (LSC) edge node[sloped,above] {\cite{park2022lsc}} (BVC);
    \path[->] (DMPC) edge node[sloped,anchor=south,pos=0.65] {\cite{park2022lsc}} (BVC);
    \path[->,bend left=1] (LSC) edge node[sloped,anchor=south,pos=0.37] {\cite{park2022lsc}} (Ego-swarm);
    \path[->] (TASC) edge node[sloped,below] {\cite{toumieh2022tasc,toumieh2023tasc}} (MADER);
    \path[->] (TASC) edge node[sloped,below] {\cite{toumieh2023tasc}} (Ego-swarm);
    \path[<->] (TASC) edge node[sloped,above] {\cite{toumieh2023tasc}} (RMADER);
    \path[->, bend left=26] (TASC) edge node[sloped,above] {\cite{toumieh2022tasc}} (LSC);
    \path[->] (DREAM) edge node[sloped, above] {} (SBC);
    \path[->] (DREAM) edge node[sloped, above] {} (RLSS);
    \path[->] (DREAM) edge node[sloped, above] {} (RMADER);

\end{tikzpicture}
}
\caption{\textbf{Collision avoidance and deadlock-free navigation comparisons of listed state-of-the-art decentralized multi-robot navigation decision-making algorithms.} A directed edge means that the source algorithm is shown to be better than the destination algorithm in experiments with multiple robots in some environments. TASC and RMADER are shown to have a similar collision avoidance performance in~\cite{toumieh2023tasc}. We compare our algorithm DREAM to SBC~\cite{wang2017safety}, RLSS~\cite{senbaslar2023rlss,senbaslar2021rlss}, and RMADER~\cite{kondo2023rmader} and establish that it results in a better performance than the listed state-of-the-art baselines.}
\label{Figure:BaselineComparison}
\vspace{-0.25in}
\end{figure}

\subsubsection{Using RMADER in Comparisons}

RMADER is a real-time decentralized trajectory planning algorithm for static and dynamic obstacle and multi-robot collision avoidance.
It explicitly accounts for asynchronous planning between robots using communication and accounts for communication delays with known bounds.
It models dynamic obstacle movements using predicted trajectories.
It does not explicitly model robot--dynamic obstacle interactions.\looseness=-1

Dynamic obstacles move according to movement and interaction models in our formulation.
We convert movement models to predicted trajectories by propagating dynamic obstacles' positions according to the desired velocities from the movement models.
Since the interactive behavior of dynamic obstacles depend on the trajectory that the robot is going to follow, which is computed by the planner, their effect on the future trajectories is unknown prior to planning.
Since RMADER does not model interactions, we do not use interactive obstacles during baseline comparisons.
RMADER supports uncertainty associated with predicted trajectories using axis aligned boxes, such that it requires the samples of the real trajectory dynamic obstacle is going to follow are contained in known bounding boxes around the samples of the predicted trajectory.
However, using uncertainty boxes resulted in a lower success rate for RMADER for different choices of uncertainty box sizes in all of our experiments, because of which we run RMADER with no uncertainty box.\looseness=-1



We use the code of RMADER published by its authors~\cite{githubGitHubMitaclrmader} and integrate it to our simulation system.
We set the desired maximum planning time of RMADER to $\SI{500}{ms}$ in each planning iteration, which it exceeds if needed, even when the simulated replanning period is smaller.
We freeze the environment until RMADER is done replanning to cancel the effects of exceeding the replanning period.
Our prediction system generates three behavior models for each dynamic obstacle.
RMADER does not support multiple behavior hypotheses explicitly.
Therefore, it has the choice of avoiding the most likely or all behavior models of dynamic obstacles, modeling each behavior model as a separate obstacle.
We provide only the most likely behavior models to RMADER during evaluation because avoiding all behavior models resulted in highly conservative behavior.\looseness=-1

\subsubsection{Using RLSS in Comparisons}

RLSS is a real-time decentralized trajectory planning algorithm for static obstacles and multi-robot collision avoidance.
It does not account for asynchronous planning between teammates.
It does not utilize communication and depends on position/shape sensing only.
When using RLSS in comparisons, we model dynamic obstacles as robots and provide their current positions and shapes to the planning algorithm.
We use our own implementation of RLSS during our comparisons.\looseness=-1

\subsubsection{Using SBC in Comparisons}

SBC is a safety barrier certificates-based safe controller for static and dynamic obstacle and multi-robot collision avoidance.
SBC runs at a high frequency ($>\SI{50}{Hz}$), therefore it does not need to account for asynchronous planning.
It does not utilize communication and depends on position, velocity, and shape sensing.
When simulating SBC, we do not use the predicted behavior models and feed the current positions and velocities of the dynamic obstacles to the algorithm, which assumes that the dynamic obstacles execute the same velocity for the short future ($<\SI{20}{ms}$).\looseness=-1

SBC models shapes of obstacles and robots as hyperspheres.
We provide shapes of objects to SBC as the smallest hyperspheres containing the axis-aligned boxes.
We sample robot sizes in the interval $[\SI{0.1}{m}, \SI{0.2}{m}]$ instead of $[\SI{0.2}{m}, \SI{0.3}{m}]$ in SBC runs so that the robot can fit between static obstacles easily when its collision shape is modeled using bounding hyperspheres. (Since the resolution of octrees we generate is $\SI{0.5}{m}$, the smallest possible gap between static obstacles is $\SI{0.5}{m}$.)
We use our own implementation of SBC.\looseness=-1

\subsubsection{Single Robot Experiments}

\begin{table*}
    \centering
    \caption{Baseline comparisons in single robot scenarios}
    \vspace{-0.05in}
    \label{Table:SingleRobotComparisons}
    \resizebox{0.8\linewidth}{!}{\begin{tabular}{|c|c|c|c|c|c|c|c|c|c|c|c|}
        \hline $\rho$ & $\#^D$ & Alg. & \textbf{succ. rate} & \textbf{coll. rate} & \textbf{deadl. rate} & \textbf{s. coll. rate} & \textbf{d. coll. rate} & \textbf{t. coll. rate} & \textbf{avg. nav. dur. [s]} & \textbf{pl. fail rate} & \textbf{avg. pl. dur. [ms]}\\
        \hline \multirow{4}{*}{$0.0$} & \multirow{4}{*}{$15$} & SBC & {\color{red}0.992} & 0.008 & 0.000 & {\color{lightgray}0.000} & 0.008 & {\color{lightgray}0.000} & 33.96 & 0.001 & 0.26\\
        &&RMADER & 0.960 & 0.036 & 0.008 & {\color{lightgray}0.000} & 0.036  &  {\color{lightgray}0.000} & 25.31 & 0.100 & 19.24\\
        &&RLSS & 0.776 & 0.224 & 0.012 & {\color{lightgray}0.000} & 0.224 & {\color{lightgray}0.000} & 26.80 & 0.052 & 6.69\\
        &&DREAM & 0.984 & 0.016 & {\color{Dandelion}0.000} & {\color{lightgray}0.000} & 0.016 & {\color{lightgray}0.000} & 26.67 & 0.008 & 13.93\\
        \hline  \multirow{4}{*}{$0.1$} & \multirow{4}{*}{$15$} & 
        SBC & {\color{red}0.352} & 0.648 & {\color{ForestGreen}0.236} & {\color{cyan}0.648} & {\color{orange}0.004} & {\color{lightgray}0.000} & 34.46 & 0.006 & 0.73\\
        &&RMADER & 0.756 &  0.096 &  {\color{ForestGreen}0.212} & {\color{magenta}0.000} & 0.096 & {\color{lightgray}0.000} & 38.17 & {\color{brown}0.313} & 341.14\\
        &&RLSS & 0.816 & 0.184 & 0.008  & {\color{magenta}0.000} &  0.184 & {\color{lightgray}0.000} & 27.26 & 0.034 & 28.37\\
        &&DREAM & 0.984 & 0.016 & {\color{Dandelion}0.000} & {\color{magenta}0.000} &0.016 & {\color{lightgray}0.000} & 27.47 & 0.025 &  21.09\\
        \hline \multirow{4}{*}{$0.2$} & \multirow{4}{*}{$15$} & SBC & 0.072 & 0.928 & {\color{ForestGreen}0.564} & {\color{cyan}0.928} & {\color{orange}0.044}  & {\color{lightgray}0.000} & 35.25 & 0.011 & 1.47\\
        &&RMADER&  0.456 & 0.204 & {\color{ForestGreen}0.516} & {\color{magenta}0.000} & 0.204 & {\color{lightgray}0.000} & 49.74 & {\color{brown}0.470} & 769.83\\
        &&RLSS & 0.780 & 0.204 & 0.036 & {\color{magenta}0.000} & 0.204 &  {\color{lightgray}0.000}  &  28.36 & 0.088 & 48.07\\
        &&DREAM & 0.988 & 0.012  & {\color{Dandelion}0.000} & {\color{magenta}0.000} & 0.012 & {\color{lightgray}0.000}  & 28.31 & 0.043 & 25.56\\
        \hline \multirow{4}{*}{$0.2$} & \multirow{4}{*}{$25$} & SBC & 0.080 & 0.920 & {\color{ForestGreen}0.560}  & {\color{cyan}0.920}  & {\color{orange}0.068} & {\color{lightgray}0.000} & 37.09 & 0.019 & 1.47\\
        &&RMADER & 0.400 & 0.264 & {\color{ForestGreen}0.568} & {\color{magenta}0.000} & 0.264 &  {\color{lightgray}0.000} &  50.77 & {\color{brown}0.535} & 751.45\\
        &&RLSS & 0.752 & 0.228 & 0.044 & {\color{magenta}0.000} & 0.228 & {\color{lightgray}0.000}  & 28.66 & 0.095 & 35.59 \\
        &&DREAM & 0.956 & 0.040  & {\color{Dandelion}0.008} & {\color{magenta}0.000} & 0.040 & {\color{lightgray}0.000} & 28.28 & 0.059 & 31.54\\
        \hline \multirow{4}{*}{$0.2$} & \multirow{4}{*}{$50$} & SBC & 0.056 & 0.944 & {\color{ForestGreen}0.556} & {\color{cyan}0.944} & {\color{orange}0.144} &  {\color{lightgray}0.000} & 42.23 &  0.032 & 1.42\\
        &&RMADER& 0.116 & 0.632 & {\color{ForestGreen}0.844} &  {\color{magenta}0.000} & 0.632 & {\color{lightgray}0.000}   & 49.79 & {\color{brown}0.755} & 890.89\\
        &&RLSS & 0.536 & 0.448 & 0.064 & {\color{magenta}0.000} & 0.448 & {\color{lightgray}0.000}& 29.33 & 0.151 & 52.93\\
        &&DREAM & 0.900 & 0.100 & {\color{Dandelion}0.008} & {\color{magenta}0.000} & 0.100 & {\color{lightgray}0.000} & 28.46 & 0.073 & 47.02 \\
        \hline \multirow{4}{*}{$0.3$} & \multirow{4}{*}{$50$} & SBC & {\color{RedViolet}0.008} & 0.992 & {\color{ForestGreen}0.788} & {\color{cyan}0.992} & {\color{orange}0.212} & {\color{lightgray}0.000}  & 48.75 & 0.041 & 1.92\\
        &&RMADER & {\color{RedViolet}0.092} & 0.536 & {\color{ForestGreen}0.904} & {\color{magenta}0.000} & 0.536 & {\color{lightgray}0.000} & 56.68 & {\color{brown}0.745} & 1217.24\\
        &&RLSS & {\color{RedViolet}0.536} & 0.436 & {\color{RoyalPurple}0.160} & {\color{magenta}0.000} & 0.436 & {\color{lightgray}0.000} &  31.28 & 0.237 & 171.81\\
        &&DREAM & {\color{RedViolet}0.884} & 0.116 & {\color{Dandelion}0.000} & {\color{magenta}0.000} & 0.116 & {\color{lightgray}0.000} & 30.25 & 0.080 & 48.10\\
        \hline
    \end{tabular}}
    \vspace{-0.2in}
\end{table*}

We compare DREAM with the baselines in environments with different static obstacle densities $\rho$ and number of dynamic obstacles $\#^D$ in single-robot experiments. The results are summarized in Table~\ref{Table:SingleRobotComparisons}.\looseness=-1


SBC's performance decreases sharply when static obstacles are introduced to the environment as can be seen in {\color{red}red} values in \textbf{succ. rate} column.
SBC mainly suffers from collisions with static obstacles compared to dynamic obstacles ({\color{cyan}cyan} vs {{\color{orange}orange} values).
All medium horizon algorithms can avoid static obstacles perfectly ({\color{magenta}magenta} values).\looseness=-1

SBC and RMADER result in a high deadlock rate as it can be seen in {\color{ForestGreen}green} values.
The reason SBC results in a high ratio of deadlocks is its short horizon decision-making setup.
Since it does not consider the longer horizon effects of the generated actions, as the environment density increases, it tends to deadlock.
RMADER results in a high ratio of planning failures as the density increases, as can be seen in {\color{brown}brown} values in the \textbf{pl. fail rate} column, which causes it to consume the generated plans and not be able to generate new ones, which results in deadlocks.
RLSS has better deadlock avoidance compared to SBC and RMADER, but it too results in deadlocks as the environment density increases, as can be seen in {\color{RoyalPurple}purple} value in \textbf{deadl. rate} column.
DREAM results in little to no deadlocks ({\color{Dandelion}yellow} values).\looseness=-1

In terms of the success rate, DREAM considerably improves the state-of-the-art, resulting in $\approx$110x improvement over SBC, $\approx$9.6x improvement over RMADER, and $\approx$1.65x improvement over RLSS in the hardest scenario ({\color{RedViolet}violet} values in \textbf{succ. rate} column}).\looseness=-1


\subsubsection{Multi Robot Experiments}

\begin{table*}
    \centering
    \caption{Baseline comparisons in multi robot scenarios with $\delta=\SI{1}{s}$ average delay and $\kappa=0.25$ message drop probability for our algorithm and $\delta=\SI{1}{s}$ maximum delay and no message drops for RMADER. SBC and RLSS do not depend on communication.}
    \vspace{-0.05in}
    \label{Table:MultiRobotComparisons}
    \resizebox{0.8\linewidth}{!}{\begin{tabular}{|c|c|c|c|c|c|c|c|c|c|c|c|c|}
        \hline  $\rho$ & $\#^D$ & $\#^R$ & Alg. & \textbf{succ. rate} & \textbf{coll. rate} & \textbf{deadl. rate} & \textbf{s. coll. rate} & \textbf{d. coll. rate} & \textbf{t. coll. rate} & \textbf{avg. nav. dur. [s]} & \textbf{pl. fail rate} & \textbf{avg. pl. dur. [ms]}\\
        \hline \multirow{4}{*}{$0.2$} & \multirow{4}{*}{$25$} & \multirow{4}{*}{$16$} & SBC  & {\color{red}0.059} & 0.941 & 0.483 & 0.941 & 0.126 & 0.004 & 38.82  & 0.021   & 1.58\\
        &&&RMADER& 0.116 & 0.551 & 0.870 & {\color{orange}0.000} & 0.551 & 0.001 & 60.90 & {\color{magenta}0.932} & 813.71 \\
        &&&RLSS & 0.488 & 0.508 & 0.039 & {\color{orange}0.000} & 0.428 & {\color{brown}0.134} & 34.95 & 0.259 & 97.78 \\
        &&&DREAM & \textbf{0.960} & 0.036 & 0.005 & 0.003 & {\color{RedViolet}0.035} & 0.000 & 40.77 &  0.014 & 123.68 \\
        \hline \multirow{4}{*}{$0.2$} & \multirow{4}{*}{$25$}& \multirow{4}{*}{$32$} & 
        SBC & {\color{red}0.061} & 0.938 & 0.450 & 0.937 & 0.121 & 0.003 & 39.13 & 0.018 & 1.62 \\
        &&&RMADER & 0.112 & 0.520 & 0.868 & {\color{orange}$<$ 0.001} & 0.514 & 0.008 &  62.82 & {\color{magenta}0.923} & 912.37\\
        &&&RLSS & 0.471 & 0.521 & 0.059 & {\color{orange}0.000} & 0.373 & {\color{brown}0.225} & 35.53 & 0.354  & 90.16\\
        &&&DREAM & \textbf{0.841} & 0.097 & 0.069 & 0.006 & {\color{RedViolet}0.074} & 0.026 & 60.31  & 0.014 & 185.87\\
        \hline \multirow{4}{*}{$0.3$} & \multirow{4}{*}{$25$}& \multirow{4}{*}{$16$} & SBC & {\color{red}0.006} & 0.994 & 0.691 & 0.994 &  0.184 & 0.001 & 42.68 & 0.028 & 1.78\\
        &&&RMADER & 0.054 & 0.534 & {\color{cyan}0.935} & {\color{orange}0.000} & 0.533 & 0.003 & 68.67  & {\color{magenta}0.952} & 1457.58\\
        &&&RLSS & 0.591 & 0.396 & 0.056 & {\color{orange}0.000} & 0.332 & {\color{brown}0.093} & 38.47 & 0.324  & 300.33 \\
        &&&DREAM& \textbf{0.941} & 0.051 & 0.011 & 0.005 & {\color{RedViolet}0.044} & 0.003  & 42.53 & 0.020  & 120.33 \\
        \hline  \multirow{4}{*}{$0.3$} & \multirow{4}{*}{$25$}& \multirow{4}{*}{$32$} & SBC & {\color{red}0.006} & 0.993 & 0.669 & 0.993 & 0.191 & 0.006 & 40.40 & 0.031 & 1.79 \\
        &&&RMADER & 0.056  & 0.484 & {\color{cyan}0.934} & {\color{orange}0.000} & 0.481 & 0.006 & 64.39   & {\color{magenta}0.889} & 1467.52\\
        &&&RLSS & 0.503 & 0.486 & 0.064 & {\color{orange}0.000} & 0.343 & {\color{brown}0.209} & 39.10 & 0.440  & 282.33 \\
        &&&DREAM & \textbf{0.782} & 0.108 & 0.127 & 0.014 & {\color{RedViolet}0.081} & 0.025 & 58.04 & 0.019  & 178.71 \\
        \hline \multirow{4}{*}{$0.3$} & \multirow{4}{*}{$50$}& \multirow{4}{*}{$16$} & SBC & {\color{red}0.006} & 0.994 & 0.636 & 0.993 & 0.318 & 0.003 & 47.60 & 0.046  & 1.96\\
        &&&RMADER & 0.041 & 0.673 & {\color{cyan}0.956} & {\color{orange}0.000} & 0.673 & 0.001 &  62.97  & {\color{magenta}0.920} & 1315.97\\
        &&&RLSS & 0.402 &  0.592 & 0.073 & {\color{orange}0.000} & 0.533 & {\color{brown}0.116} &  39.15 & 0.375   & 354.34\\
        &&&DREAM & \textbf{0.864} & 0.128 & 0.012 & 0.006 & {\color{RedViolet}0.122} & 0.008 & 43.30 & 0.023   & 128.87 \\
        \hline \multirow{4}{*}{$0.3$} & \multirow{4}{*}{$50$}& \multirow{4}{*}{$32$} & SBC & {\color{red}0.006} & 0.993 & 0.646 & 0.992 & 0.348 & 0.011 &  44.82 & 0.050 & 2.01\\
        &&&RMADER & 0.022 & 0.701 & {\color{cyan}0.974} & {\color{orange}$<$0.001} & 0.699 & 0.005 & 63.12   & {\color{magenta}0.916}  &1109.09 \\
        &&&RLSS & 0.357 & 0.634 & 0.098 & {\color{orange}0.000} & 0.540 & {\color{brown}0.197} & 39.77 & 0.475 & 325.69\\
        &&&DREAM & \textbf{0.712} & 0.216 & 0.093 & 0.018 & {\color{RedViolet}0.183} &  0.046 & 59.46 & 0.028   & 189.06 \\
        \hline
    \end{tabular}}
    \vspace{-0.2in}
\end{table*}

We compare DREAM with the baselines in highly cluttered environments with different $\rho$, $\#^D$ and $\#^R$.
During these experiments, we simulate communication imperfections.
SBC and RLSS do not depend on communication and hence communication imperfections do not affect them.
RMADER accounts for message delays with known bounds.
DREAM accounts for message delays with unknown bounds as well as message drops.\looseness=-1

For DREAM, we introduce mean communication delay $\delta = {1}{s}$ and message drop probability $\kappa=0.25$.
Since RMADER does not account for message drops, we set $\kappa = 0.0$ for RMADER.
In addition, in RMADER, we bound communication delays with the mean $\delta$ by generating a random sample from the distribution and setting it to $\delta$ if it is more than $\delta$.
Therefore, DREAM runs in considerably more challenging environments during these experiments compared to RMADER.
RMADER has a \emph{delay check} phase to account for communication delays, which should run for at least the maximum communication delay. We set its duration to $\SI{1.1}{s}$.
The environments used are more challenging compared to single-robot experiments, as not only $\rho$ and $\#^D$ are high, but also multiple teammates navigate under asynchronous decision making and considerable communication imperfections.\looseness=-1

The results are summarized in Table~\ref{Table:MultiRobotComparisons}.
SBC is ineffective for navigating in environments with high clutter ({\color{red}red} values in \textbf{succ. rate} column).
Since the communication imperfections are high, RMADER results in conservative behavior, resulting in deadlocks.
Given that it already results in a considerable rate of deadlocks in single-robot scenarios, almost all robots using RMADER deadlock in the hardest scenarios ({\color{cyan}cyan} values in \textbf{deadl. rate} column).
The high planning failure rate of RMADER ({\color{magenta}magenta} values in \textbf{pl. fail rate} column) is the main cause of the deadlocks: once plans are consumed and planning fails, it keeps failing until the end of the simulation.\looseness=-1

Both RLSS and RMADER result in no collisions with static obstacles as they i) avoid static obstacles using hard constraints unlike DREAM, and ii) enforce static obstacle avoidance for the full horizon unlike SBC ({\color{orange}orange} values in \textbf{s. coll. rate} column).
DREAM results in the lowest dynamic obstacle avoidance rate ({\color{RedViolet}violet} values in \textbf{d. coll. rate} column).
RLSS results in high teammate collisions, because it is the only algorithm that is affected by asynchronous planning but does not account for it ({\color{brown}brown} values in \textbf{t. coll. rate} column).\looseness=-1

In terms of success rate, DREAM considerably improves the state-of-the-art, resulting in $\approx$156.8x improvement over SBC, $\approx$32.36x improvement over RMADER, $\approx$2.15x improvement over RLSS (\textbf{bold} values in \textbf{succ. rate} column).\looseness=-1

\subsection{Physical Robot Experiments}

\begin{figure}
    \centering
    \subfloat[A teammate navigating through six rotating not interactive obstacles.]{\includegraphics[width=0.49\linewidth]{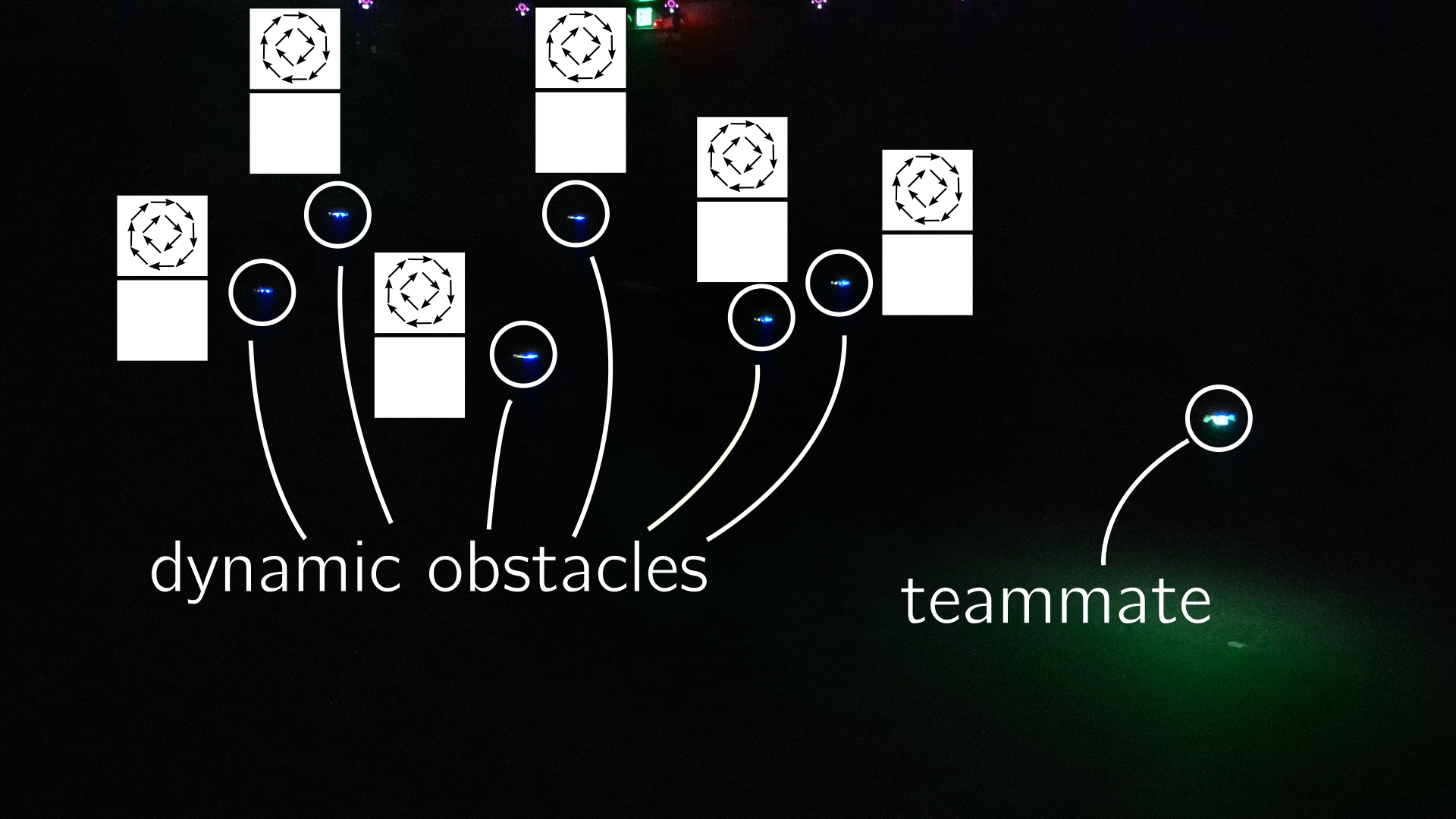}}\hfill
    \subfloat[A teammate navigating through six constant velocity repulsive obstacles.]{\includegraphics[width=0.49\linewidth]{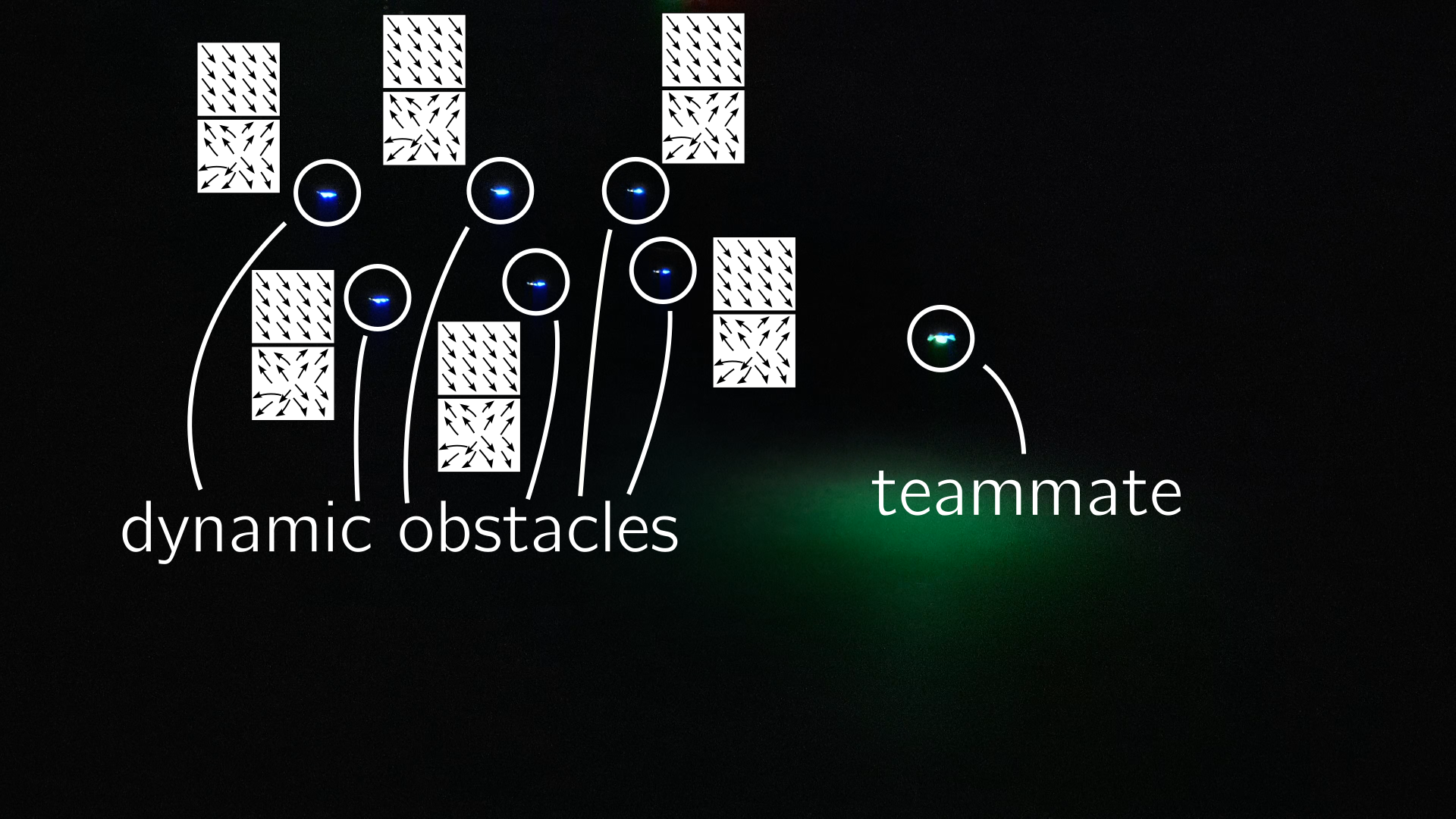}}\\
    \vspace{-0.1in}
    \subfloat[A teammate navigating through three goal attractive repulsive and three goal attractive not interactive obstacles.]{\includegraphics[width=0.49\linewidth]{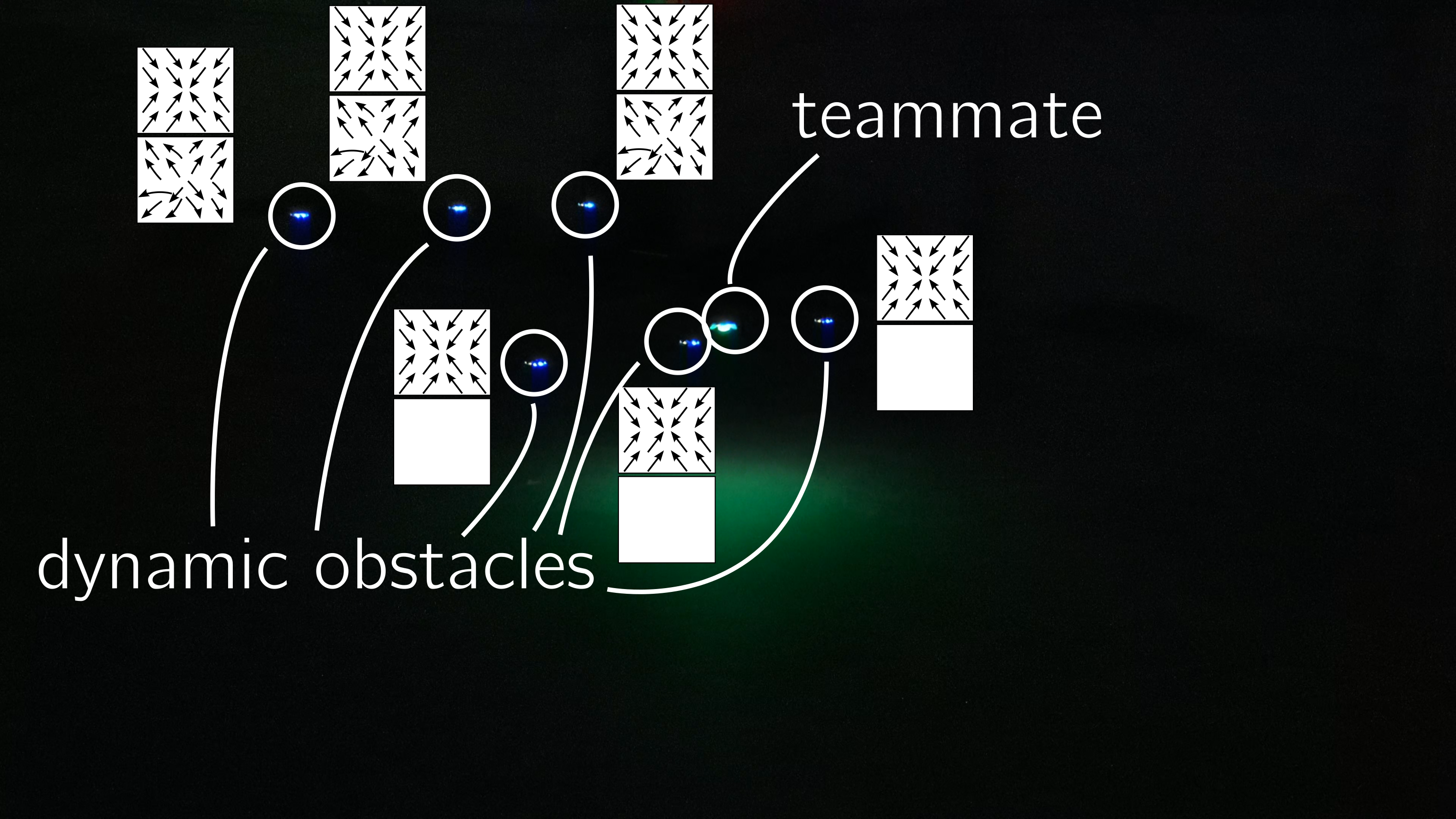}}\hfill
    \subfloat[Four teammates navigating through three rotating not interactive dynamic obstacles and a static obstacle.]{\includegraphics[width=0.49\linewidth]{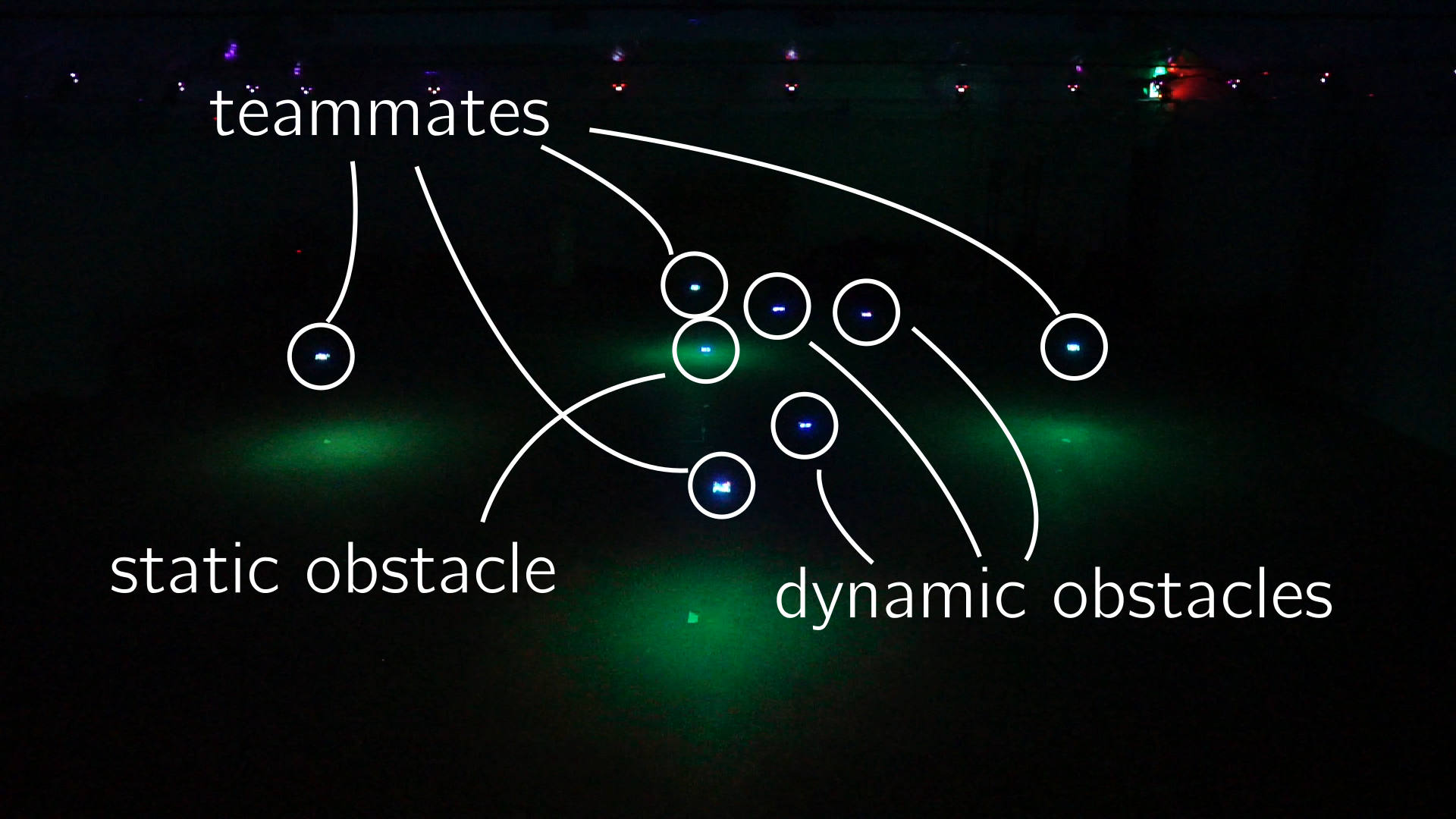}}
    \caption{\textbf{Pictures from physical robot experiments.} We implement our algorithm for physical quadrotor flight. We conduct single and multi-robot experiments to show the real-world applicability of our algorithm in various environments. In our physical robot experiments, teammates navigate through the environments without collisions or deadlocks. The recordings of the physical robot experiments can be found in the supplemental video.}
    \vspace{-0.2in}
    \label{Figure:PhysicalExperiments}
\end{figure}

We implement and run DREAM for Crazyflie 2.1 quadrotors.
We use quadrotors as i) dynamic obstacles moving according to goal attractive, rotating, or constant velocity movement models and repulsive interaction model, ii) static obstacles, and iii) teammates navigating to their goal positions using our planner.
Each planning quadrotor runs the predictors in real time to generate a probability distribution over behavior models of each dynamic obstacle.
Then, it runs our planner to compute trajectories in real time.\looseness=-1

For obstacle and robot localization, we use VICON motion tracking system, and we manage the Crazyflies using Crazyswarm~\cite{preiss2017crazyswarm}.
We use Robot Operating System (ROS) as the underlying software system.
Predictors and our algorithm run on a separate process for each robot in a base station computer with Intel(R) Xeon(R) CPU E5-2630 v4 @2.20GHz CPU, running Ubuntu 20.04 as the operating system.\looseness=-1

Pictures from the physical experiments can be seen in Fig.~\ref{Figure:PhysicalExperiments}.
The recordings from our physical experiments can be found in our supplemental video.
Our physical experiments show the feasibility of running DREAM in real-time in the real world.
\looseness=-1

\section{Conclusion}
We present DREAM--a decentralized multi-robot real-time trajectory planning algorithm for mobile robots navigating in environments with static and interactive dynamic obstacles.
DREAM explicitly minimizes the probabilities of collision with static and dynamic obstacles and violations of discretized separating hyperplane trajectories with respect to teammates as well as distance, duration, and rotations using a multi-objective search method; and energy usage during optimization.
The behavior of dynamic obstacles is modeled using two vector fields, namely movement and interaction models.
DREAM simulates the behaviors of dynamic obstacles during decision-making in response to the actions the planning robot is going to take using the interaction models.
We present three online model-based prediction algorithms to predict the behavior of dynamic obstacles and assign probabilities to them.
We extensively evaluate DREAM in different environments and configurations and compare with three state-of-the-art decentralized real-time multi-robot navigation decision-making methods. DREAM considerably improves the state-of-the-art, achieving up to 1.68x success rate using as low as 0.28x time in single-robot, and up to 2.15x success rate using as low as 0.36x time in multi-robot experiments compared to the best baseline.
We show its feasibility in the real-world by implementing and running it for quadrotors.\looseness=-1

Future work includes modeling inter-dynamic obstacle interactions during decision-making.\looseness=-1

\bibliographystyle{IEEEtran}
\bibliography{bibliography}



\end{document}